\def\year{2021}\relax
\documentclass[letterpaper]{article} 
\usepackage{aaai21}  
\usepackage{times}  
\usepackage{helvet} 
\usepackage{courier}  
\usepackage[hyphens]{url}  
\usepackage{graphicx} 
\usepackage{natbib}  
\usepackage{caption} 
\usepackage{pgfplots}
\pgfplotsset{compat=1.17}
\usepackage{tikz}
\usepackage{algorithm}
\usepackage{algorithmic}
\usepackage{multirow}
\usetikzlibrary{patterns}
\usepackage{hyperref}

\usepackage{amsmath,amsfonts,amsthm,bm, amssymb} 
\usepackage{enumitem}

\makeatletter
\newcommand{\axissize}{\@setfontsize{\axissize}{5pt}{5pt}}
\makeatother

\makeatletter
\newcommand{\Labelsize}{\@setfontsize{\Labelsize}{5pt}{5pt}}
\makeatother

\newtheorem{definition}{Definition}
\newtheorem{theorem}{Theorem}
\newtheorem{corollary}{Corollary}[theorem]
\newtheorem{lemma}{Lemma}[theorem]

\newcommand{\norm}[1]{\left\lVert#1\right\rVert}

\DeclareMathOperator*{\argmin}{arg\,min}
\DeclareMathOperator{\Tr}{Tr}

\usepackage{array}
\newcolumntype{P}[1]{>{\centering\arraybackslash}p{#1}}

\pgfplotsset{mystyle2/.style={%
        font=\rmfamily\Labelsize,
        width=0.227\textwidth,
        mark size=1.1pt,
        ylabel near ticks,
        xlabel near ticks,
        label style = {font=\scriptsize},
        tick label style = {font=\scriptsize, yshift=0.5ex},
        ylabel shift = -6 pt, 
        xlabel shift = -5 pt,
        title style={yshift=-1.2ex,font=\footnotesize},
        y tick label style={/pgf/number format/.cd,fixed,fixed zerofill,precision=2,/tikz/.cd},
        legend image post style={scale=0.5},
        every axis plot/.append style={semithick},
        legend style={font=\scriptsize, mark size=4pt},
        legend columns=1, legend style={/tikz/column 2/.style={column sep=0.5pt}},
        mark options={scale=1.7}}}

\newcommand\blfootnote[1]{%
  \begingroup
  \renewcommand\thefootnote{}\footnote{#1}%
  \addtocounter{footnote}{-1}%
  \endgroup
}

\title{Detecting Beneficial Feature Interactions for Recommender Systems}
\author{Yixin Su, \textsuperscript{\rm 1}
Rui Zhang, \textsuperscript{\rm 2}
Sarah Erfani, \textsuperscript{\rm 1}
Zhenghua Xu \textsuperscript{\rm 3}\blfootnote{Rui Zhang and Zhenghua Xu are corresponding authors.}\\
}

\affiliations{
\textsuperscript{\rm 1} University of Melbourne  \\
\textsuperscript{\rm 2} Tsinghua University   \\
\textsuperscript{\rm 3} Hebei University of Technology  \\
yixins1@student.unimelb.edu.au, rui.zhang@ieee.org, sarah.erfani@unimelb.edu.au, zhenghua.xu@hebut.edu.cn
}
\begin{document}

\maketitle

\begin{abstract}
\begin{quote}
Feature interactions are essential for achieving high accuracy in recommender systems.
Many studies take into account the interaction between every pair of features. 
However, this is suboptimal because some feature interactions may not be that relevant to the recommendation result, and taking them into account may introduce noise and decrease recommendation accuracy.
To make the best out of feature interactions, we propose a graph neural network approach to effectively model them, together with a novel technique to automatically detect those feature interactions that are beneficial in terms of recommendation accuracy.
The automatic feature interaction detection is achieved via edge prediction with an $L_0$ activation regularization. 
Our proposed model is proved to be effective through the information bottleneck principle and statistical interaction theory.
Experimental results show that our model (i) outperforms existing baselines in terms of accuracy, and (ii) automatically identifies beneficial feature interactions.
\end{quote}
\end{abstract}

\section{Introduction}
Recommender systems play a central role in addressing information overload issues in many Web applications, such as e-commerce, social media platforms, and lifestyle apps. The core of recommender systems is to predict how likely a user will interact with an item (e.g., purchase, click). 
An important technique is to discover the effects of features (e,g., contexts, user/item attributes) on the target prediction outcomes for fine-grained analysis \cite{shi2014collaborative}. 
Some features are correlated to each other, and the joint effects of these correlated features (i.e., \textit{feature interactions}) are crucial for recommender systems to get high accuracy \cite{blondel2016polynomial,he2017neural}.
For example, it is reasonable to recommend a user to use \textit{Uber} on a \textit{rainy} day at off-work hours (e.g., during \textit{5-6pm}). In this situation, considering the feature interaction $<\textit{5-6pm}, \textit{rainy}>$ is more effective than considering the two features separately.
Therefore, in recent years, many research efforts have been put in modeling the feature interactions \cite{he2017neural,lian2018xdeepfm,song2019autoint}.
These models take into account the interaction between every pair of features. However, in practice, not all feature interactions are relevant to the recommendation result \cite{langley1994selection,siegmund2012predicting}. Modeling the feature interactions that provide little useful information may introduce noise and cause overfitting, and hence decrease the prediction accuracy \cite{zhang2016understanding,louizos2017learning}. 
For example, a user may use \textit{Gmail} on a workday no matter what weather it is. 
However, if the interaction of workday and specific weather is taking into account in the model and due to the bias in the training set (e.g., the weather of the days when the Gmail usage data are collected happens to be cloudy), the interaction $<\textit{Monday}, \textit{Cloudy}>$ might be picked up by the model and make less accurate recommendations on Gmail usage.
Some work considers each feature interaction's importance through the attention mechanism \cite{xiao2017attentional,song2019autoint}. However, these methods still take all feature interactions into account.  
Moreover, they capture each \textit{individual} feature interaction's contribution to the recommendation prediction, failing to capture the \textit{holistic} contribution of a set of feature interactions together.

In this paper, we focus on identifying the set of feature interactions that together produce the best recommendation performance. 
To formulate this idea, we propose the novel problem of \textit{detecting beneficial feature interactions}, which is defined as identifying the set of feature interactions that contribute most to the recommendation accuracy (see Definition 1 for details).
Then, we propose a novel graph neural network (GNN)-based recommendation model, $L_0$-SIGN, that detects the most beneficial feature interactions and utilizes only the beneficial feature interactions for recommendation, where each data sample is treated as a graph, features as nodes and feature interactions as edges.
Specifically, our model consists of two components.
One component is an $L_0$ edge prediction model, which detects the most beneficial feature interactions by predicting the existence of edges between nodes.
To ensure the success of the detection, an $L_0$ activation regularization is proposed to encourage unbeneficial edges (i.e. feature interactions) to have the value of $0$, which means that edge does not exist. Another component is a graph classification model, called \textbf{S}tatistical \textbf{I}nteraction \textbf{G}raph neural \textbf{N}etwork (SIGN). SIGN takes nodes (i.e., features) and detected edges (i.e., beneficial feature interactions) as the input graph, and outputs recommendation predictions by effectively modeling and aggregating the node pairs that are linked by an edge.

Theoretical analyses are further conducted to verify the effectiveness of our model. First, the most beneficial feature interactions are guaranteed to be detected in $L_0$-SIGN. This is proved by showing that the empirical risk minimization procedure of $L_0$-SIGN is a variational approximation of the Information Bottleneck (IB) principle, which is a golden criterion to find the most relevant information correlating to target outcomes from inputs \cite{tishby2000information}. Specifically, only the most beneficial feature interactions will be retained in $L_0$-SIGN. It is because, in the training stage, our model simultaneously minimizes the number of detected feature interactions by the $L_0$ activation regularization, and maximizes the recommendation accuracy with the detected feature interactions.
Second, we further show that the modeling of the detected feature interactions in SIGN is very effective. By accurately leveraging the relational reasoning ability of GNN, \textit{iff} a feature interaction is detected to be beneficial in the $L_0$ edge prediction component, it will be modeled in SIGN as a statistical interaction (an interaction is called statistical interaction if the joint effects of variables are modeled correctly).

We summarize our contributions as follows:
\begin{itemize}
\item This is the first work to formulate the concept of beneficial feature interactions, and propose the problem of detecting beneficial feature interactions for recommender systems.
\item We propose a model, named $L_0$-SIGN, to detect the beneficial feature interactions via a graph neural network approach and $L_0$ regularization.
\item We theoretically prove the effectiveness of $L_0$-SIGN through the information bottleneck principle and statistical interaction theory.
\item We have conducted extensive experimental studies. The results show that (i) $L_0$-SIGN outperforms existing baselines in terms of accuracy; (ii) $L_0$-SIGN effectively identifies beneficial feature interactions, which result in the superior prediction accuracy of our model.
\end{itemize}

\section{Related Work}
\label{sec:related_work}
\subsubsection{Feature Interaction based Recommender Systems}
Recommender systems are one of the most critical research domains in machine learning \cite{lu2015recommender,wang2019doubly}.
Factorization machine (FM) \cite{rendle2010factorization} is one of the most popular algorithms in considering feature interactions. However, FM and its deep learning-based extensions \cite{xiao2017attentional,he2017neural,guo2017deepfm} consider all possible feature interactions, while our model detects and models only the most beneficial feature interactions.
Recent work considers the importance of feature interactions by giving each feature interaction an attention value \cite{xiao2017attentional,song2019autoint}, or select important interactions by using gates \cite{liu2020autofis} or by searching in a tree-structured space \cite{luo2019autocross}. In these methods, the importance is not determined by the holistic contribution of the feature interactions, thus limit the performance. Our model detects beneficial feature interactions that together produce the best recommendation performance.

\subsubsection{Graph Neural Networks (GNNs)} 
GNNs can facilitate learning entities and their holistic relations \cite{battaglia2018relational}. Existing work leverages GNNs to perform relational reasoning in various domains. For example, \citet{duvenaud2015convolutional} and \citet{gilmer2017neural} use GNNs to predict molecules' property by learning their features from molecular graphs.
\citet{chang2016compositional} use GNNs to learn object relations in dynamic physical systems. Besides, some relational reasoning models in computer vision such as \cite{santoro2017simple,wang2018non} have been shown to be variations of GNNs \cite{battaglia2018relational}. Fi-GNN \cite{li2019fi} use GNNs for feature interaction modeling in CTR prediction. However, it still models all feature interactions.
Our model theoretically connects the beneficial feature interactions in recommender systems to the edge set in graphs and leverages the relational reasoning ability of GNNs to model beneficial feature interactions' holistic contribution to the recommendation predictions.

\subsubsection{$L_0$ Regularization} 
$L_0$ regularization sparsifies models by penalizing non-zero parameters. Due to the problem of non-differentiable, it does not attract attention previously in deep learning domains until \citet{louizos2017learning} solve this problem by proposing a hard concrete distribution in $L_0$ regularization. Then, $L_0$ regularization has been commonly utilized to compress neural networks \cite{tsang2018neural,shi2019pvae,yang2017bridging}.
We explore to utilize $L_0$ regularization to limit the number of detected edges in feature graphs for beneficial feature interaction detection.

\section{Problem Formulation and Definitions}
\label{sec:prel}

Consider a dataset with input-output pairs: $D= \{(X_{n}, y_{n})\}_{1\leq n \leq N}$, where $y_{n}\in\mathbb{R/\mathbb{Z}}$, $X_{n}=\{c_k:x_k\}_{k\in J_n}$ is a set of categorical features ($c$) with their values ($x$), $J_n\subseteq J$ and $J$ is an index set of all features in $D$. 
For example, in app recommendation, $X_{n}$ consists of a user ID, an app ID and context features (e.g., \textit{Cloudy}, \textit{Monday}) with values to be $1$ (i.e., recorded in this data sample), and $y_{n}$ is a binary value to indicate whether the user will use this app. 
Our goal is to design a model $F(X_{n})$ that detects the most beneficial \textit{pairwise}\footnote{We focus on pairwise feature interactions in this paper, and we leave high-order feature interactions in future work.} feature interactions and utilizes only the detected feature interactions to predict the true output $y_n$. 

\subsubsection{Beneficial Pairwise Feature Interactions}
Inspired by the definition of relevant feature by usefulness \cite{langley1994selection,blum1997selection}, we formally define the beneficial feature interactions in Definition \ref{def:rfi}.
\begin{definition}
\label{def:rfi}
(\textbf{Beneficial Pairwise Feature Interactions}) Given a data sample $X=\{x_i\}_{1\leq i\leq k}$ of $k$ features whose corresponding full pairwise feature interaction set is $A=\{(x_i, x_j)\}_{1\leq i,j\leq k}$, a set of pairwise feature interactions $I_1 \subseteq A$ is more beneficial than another set of pairwise feature interactions $I_2 \subseteq A$ to a model with $X$ as input if the accuracy of the predictions that the model produces using $I_1$ is higher than the accuracy achieved using $I_2$.

\end{definition}

The above definition formalizes our detection goal: find and retain only a part of feature interactions that together produce the highest prediction accuracy by our model.

\subsubsection{Statistical Interaction}
Statistical interaction, or non-additive interaction, ensures a joint influence of several variables on an output variable is not additive \cite{tsang2018neural}. \citet{sorokina2008detecting} formally define the pairwise statistical interaction:
\begin{definition}
\label{def:spi}
(\textbf{Pairwise Statistical Interaction}) Function $F(X)$, where $X=\{x_i\}_{1\leq i\leq k}$ has $k$ variables, shows \textbf{no} pairwise statistical interaction between variables $x_i$ and $x_j$ if $F(X)$ can be expressed as the sum of two functions $f_{\char`\\i}$ and $f_{\char`\\j}$, where $f_{\char`\\i}$ does not depend on $x_i$ and $f_{\char`\\j}$ does not depend on $x_j$:
\begin{equation}
\label{fun:spi}
\small
\begin{split}
    F(X) = & f_{\char`\\i}(x_{1},...,x_{i-1},x_{i+1},...,x_{k}) \\
    & + f_{\char`\\j}(x_{1},...,x_{j-1},x_{j+1},...,x_{k}).
\end{split}
\end{equation}
\end{definition}
More generally, if using $\bm{v}_{i}\in\mathbb{R}^{d}$ to describe the $i$-th variable with $d$ factors \cite{rendle2010factorization}, e.g., variable embedding, each variable can be described in a vector form $\bm{u}_i=x_i\bm{v}_i$. 
Then, we define the pairwise statistical interaction in variable factor form by changing the Equation \ref{fun:spi} into:  
\begin{equation}
\label{fun:spi_emb}
\small
\begin{split}
    F(X) = & f_{\char`\\i}(\bm{u}_1,...,\bm{u}_{i-1},\bm{u}_{i+1},...,\bm{u}_{k}) \\
    & + f_{\char`\\j}(\bm{u}_{1},...,\bm{u}_{j-1},\bm{u}_{j+1},...,\bm{u}_{k}).
\end{split}
\end{equation}

The definition indicates that the interaction information of variables (features) $x_i$ and $x_j$ will not be captured by $F(X)$ if it can be expressed as the above equations. Therefore, to correctly capture the interaction information, our model should not be expressed as the above equations. In this paper, we theoretically prove that the interaction modeling in our model strictly follow the definition of pairwise statistical interaction, which ensures our model to correctly capture interaction information.

\section{Our Proposed Model}
\label{sec:method}

In this section, we formally describe $L_0$-SIGN's overview structure and the two components in detail. 
Then, we provide theoretical analyses of our model.

\subsection{$L_0$-SIGN}
\subsubsection{Model Overview}
\label{subsec:model_overview}
Each input of $L_0$-SIGN is represented as a graph (without edge information), where its features are nodes and their interactions are edges.
More specifically, a data sample $n$ is a graph $G_n(X_n, E_n)$, and $E_n=\{(e_n)_{ij}\}_{i,j\in X_n}$ is a set of edge/interaction values \footnote{In this paper, nodes and features are used interchangeably, and the same as edges and feature interactions.}, where $(e_n)_{ij}\in\{1,0\}$ , $1$ indicates that there is an edge (beneficial feature interaction) between nodes $i$ and $j$, and $0$ otherwise. Since no edge information are required, $E_n=\emptyset$.

While predicting, the $L_0$ edge prediction component, $F_{ep}(X_n;\bm{\omega})$, analyzes the existence of edges on each pair of nodes, where $\bm{\omega}$ are parameters of $F_{ep}$, and outputs the predicted edge set $E_n^{'}$. Then, the graph classification component, SIGN, performs predictions based on $G(X_n, E_n^{'})$.
Specifically, SIGN firstly conducts interaction modeling on each pair of initial node representation that are linked by an edge. Then, each node representation is updated by aggregating all of the corresponding modeling results. Finally, all updated node representations are aggregated to get the final prediction.
The general form of SIGN prediction function is $y_{n}^{'}=f_{S}(G_n(X_n, E^{'}_n);\bm{\theta})$, where $\bm{\theta}$ is SIGN's parameters and the predicted outcome $y_{n}^{'}$ is the graph classification result. Therefore, the $L_0$-SIGN prediction function $f_{LS}$ is:
\begin{equation}
\small
\label{fun:l0_sign_general}
    f_{LS}(G_n(X_n, \emptyset);\bm{\theta},\bm{\omega}) = f_{S}(G_n(X_n, F_{ep}(X_n;\bm{\omega})); \bm{\theta}).
\end{equation}

Figure \ref{fig:sign_frame} shows the structure of $L_0$-SIGN\footnote{Section \textit{D}
of Appendix lists the pseudocodes of our model and the training procedures.}. Next, we will show the two components in detail. 
When describing the two components, we focus on one input-output pair, so we omit the index $``n"$ for simplicity. 

\begin{figure}[t]
\centering
\centerline{\includegraphics[width=0.88\columnwidth]{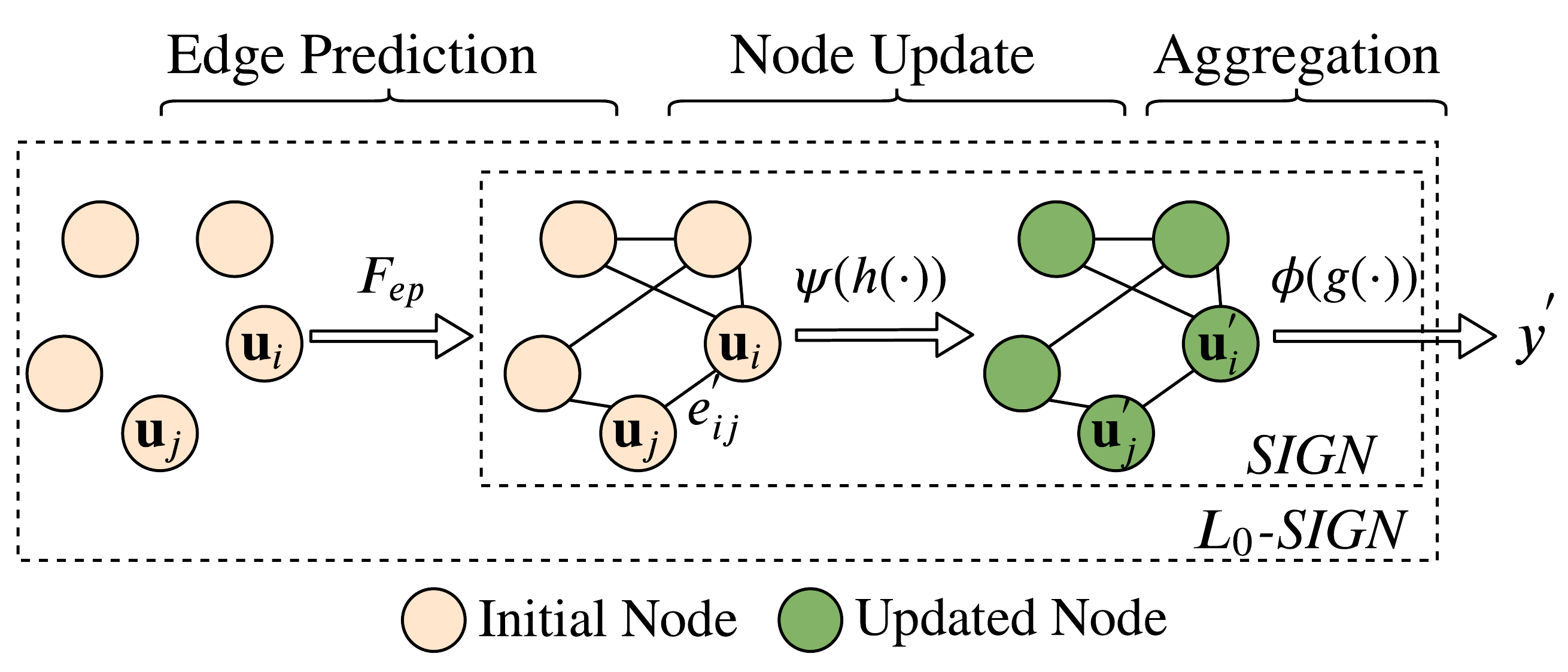}}
\caption{An overview of $L_0$-SIGN.}
\label{fig:sign_frame}
\end{figure}

\subsubsection{$L_0$ Edge Prediction Model}
\label{subsec:l0}
A (neural) matrix factorization (MF) based model is used for edge prediction. MF is effective in modeling relations between node pairs by factorizing the adjacency matrix of a graph into node dense embeddings \cite{menon2011link}. In $L_0$-SIGN, since we do not have the ground truth adjacency matrix, the gradients for optimizing this component come from the errors between the outputs of SIGN and the target outcomes.

Specifically, the edge value, $e^{'}_{ij}\in E^{'}$, is predicted by an edge prediction function $f_{ep}(\bm{v}^{e}_{i}, \bm{v}^{e}_{j}):\mathbb{R}^{2\times b}\rightarrow\mathbb{Z}_{2}$, which takes a pair of node embeddings for edge prediction with dimension $b$ as input, and output a binary value, indicating whether the two nodes are connected by an edge. $\bm{v}^{e}_{i}=\bm{o}_i\bm{W}^{e}$ is the embedding of node $i$ for edge prediction, where $\bm{W}^{e}\in\mathbb{R}^{|J|\times b}$ are parameters and $\bm{o}_i$ is the one-hot embedding of node $i$. 
To ensure that the detection results are identical to the same pair of nodes, $f_{ep}$ should be invariant to the order of its input, i.e., $f_{ep}(\bm{v}^{e}_{i}, \bm{v}^{e}_{j})=f_{ep}(\bm{v}^{e}_{j}, \bm{v}^{e}_{i})$. For example, in our experiments, we use an multilayer neural network (MLP) with the input as the element-wise product result of $\bm{v}^{e}_{i}$ and $\bm{v}^{e}_{j}$ (details are in the section of experiments).
Note that $e^{'}_{ii}=1$ can be regarded as the feature $i$ being beneficial.

While training, an $L_0$ activation regularization is performed on $E^{'}$ to minimize the number of detected edges, which will be described later in the section about the empirical risk minimization function of $L_0$-SIGN. 

\subsubsection{SIGN}
\label{subsec:sign}

In SIGN, each node $i$ is first represented as an initial node embedding $\bm{v}_i$ of $d$ dimensions for interaction modeling, i.e., each node has node embeddings $\bm{v}^{e}_{i}$ and $\bm{v}_i$ for edge prediction and interaction modeling, respectively, to ensure the best performance of respective tasks. 
Then, the interaction modeling is performed on each node pair $(i, j)$ that $e^{'}_{ij}=1$, by a non-additive function $h(\bm{u}_i, \bm{u}_j):\mathbb{R}^{2\times d}\rightarrow\mathbb{R}^{d}$ (e.g., an MLP), where $\bm{u}_i=x_i\bm{v}_i$. The output of $h(\bm{u}_i, \bm{u}_j)$ is denoted as $\bm{z}_{ij}$. 
Similar to $f_{ep}$, $h$ should also be invariant to the order of its input. The above procedure can be reformulated as $\bm{s}_{ij} = e^{'}_{ij}\bm{z}_{ij}$, where $\bm{s}_{ij}\in\mathbb{R}^{d}$ is called the statistical interaction analysis result of $(i, j)$.

Next, each node is updated by aggregating all of the analysis results between the node and its neighbors using a linear aggregation function $\psi$: $\bm{v}_{i}^{'} = \psi(\varsigma_{i})$, where $\bm{v}_{i}^{'}\in \mathbb{R}^d$ is the updated embedding of node $i$, $\varsigma_{i}$ is a set of statistical interaction analysis results between node $i$ and its neighbors. Note that $\psi$ should be invariant to input permutations, and be able to take inputs with variant
number of elements (e.g., element-wise summation/mean). 

Finally, each updated node embedding will be transformed into a scalar value by a linear function $g:\mathbb{R}^d \rightarrow \mathbb{R}$, and all scalar values are linearly aggregated as the output of SIGN. That is: $y^{'} = \phi(\nu)$, where $\nu=\{g(\bm{u}_{i}^{'})\mid i\in X\}$, $\bm{u}_{i}^{'}=x_i\bm{v}_i^{'}$ and $\phi:\mathbb{R}^{|\nu|\times 1} \rightarrow \mathbb{R}$ is an aggregation function having similar properties to $\psi$. Therefore, the prediction function of SIGN is:
\begin{equation}
\small
\label{fun:si_final_function}
    f_{S}(G; \bm{\theta}) = \phi(\{g(\psi(\{e^{'}_{ij}h(\bm{u}_{i}, \bm{u}_{j})\}_{j\in X}))\}_{i \in X}).
\end{equation}

In summary, we formulate the $L_0$-SIGN prediction function of Equation \ref{fun:l0_sign_general} with the two components in detail \footnote{The time complexity analysis is in Section \textit{E\iffalse\ref{appx:time_complexity}\fi}   of Appendix.}:
\begin{equation}
\small
\label{fun:l0_final_function}
f_{LS}(G; \bm{\omega}, \bm{\theta}) =\phi(\{g(\psi(\{f_{ep}(\bm{v}^{e}_{i}, \bm{v}^{e}_{j})h(\bm{u}_{i}, \bm{u}_{j})\}_{j\in X}))\}_{i \in X}).
\end{equation}

\subsection{Empirical Risk Minimization Function}
The empirical risk minimization function of $L_0$-SIGN minimizes a loss function, a reparameterized $L_0$ activation regularization\footnote{Section \textit{F\iffalse\ref{appx:l0_reg}\fi} of Appendix gives detailed description about $L_0$ regularization and its reparameterization trick.} on $E^{'}_{n}$, and an $L_2$ activation regularization on $\bm{z}_n$. Instead of regularizing parameters, activation regularization regularizes the output of models \cite{merity2017revisiting}. We leverage activation regularization to link our model with the IB principle to ensures the success of the interaction detection, which will be discussed in the theoretical analyses.
Formally, the function is:
\begin{equation}
\label{fun:l0_sign_loss}
\small
\begin{split}
    \mathcal{R}(\bm{\theta}, \bm{\omega})=\frac{1}{N}&\sum_{n=1}^{N}(\mathcal{L}(F_{LS}(G_n;\bm{\omega},\bm{\theta}),y_{n}) \\
    + &\lambda_1 \sum_{i,j\in X_n}(\pi_{n})_{ij} + \lambda_2 \norm{\bm{z}_{n}}_2), \\
    \bm{\theta}^{*}, \bm{\omega}^{*}=&\argmin_{\bm{\theta},\bm{\omega}} {\mathcal{R}(\bm{\theta}, \bm{\omega})}, 
\end{split}
\end{equation}
where $(\pi_{n})_{ij}$ is the probability of $(e^{'}_{n})_{ij}$ being 1 (i.e., $(e^{'}_{n})_{ij}=Bern((\pi_{n})_{ij})$), $G_n=G_n(X_n, \emptyset)$, $\lambda_1$ and $\lambda_2$ are weight factors for the regularizations, $\mathcal{L}(\cdot)$ corresponds to a loss function and $\bm{\theta}^{*}$, $\bm{\omega}^{*}$ are final parameters.
 
A practical difficulty of performing $L_0$ regularization is that it is non-differentiable. Inspired by \cite{louizos2017learning}, we smooth the $L_0$ activation regularization by approximating the Bernoulli distribution with a hard concrete distribution 
so that $e^{'}_{ij}$ is differentiable. Section \textit{G\iffalse\ref{appx:hard_concrete}\fi} of Appendix gives details about the approximation.

\subsection{Theoretical Analyses}
We conduct theoretical analyses to verify our model's effectiveness, including how $L_0$-SIGN satisfies the IB principle to guarantee the success of beneficial interaction detection, the relation of the statistical interaction analysis results with the spike-and-slab distribution (the golden standard in sparsity), and how SIGN and $L_0$-SIGN strictly follow the statistical interaction theory for effective interaction modeling.

\subsubsection{Satisfaction of Information Bottleneck (IB) Principle}
\label{subsec:relation_ib}
IB principle \cite{tishby2000information} aims to extract the most relevant information that input random variables $\bm{X}$ contains about output variables $\bm{Y}$ by considering a trade-off between the accuracy and complexity of the process. The relevant part of $\bm{X}$ over $\bm{Y}$ denotes $\bm{S}$. The IB principle can be mathematically represented as:
\begin{equation}
\small
\label{eq:ib}
\min \quad (I(\bm{X};\bm{S}) - \beta I(\bm{S};\bm{Y})),
\end{equation}
where $I(\cdot)$ denotes mutual information between two variables and $\beta > 0 $ is a weight factor.

\begin{figure}
    \centering
    \includegraphics[width=0.70\columnwidth]{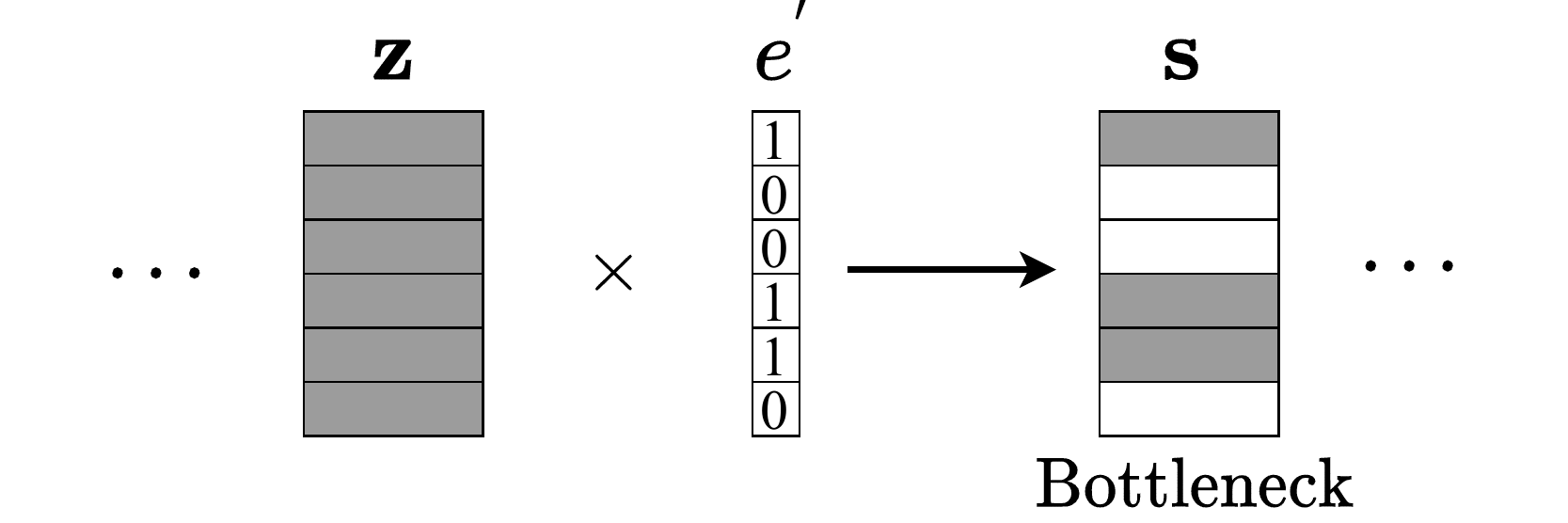}
    \caption{The interaction analysis results $\bm{s}$ from $L_0$-SIGN are like the relevant part (bottleneck) in the IB principle.}
    \label{fig:bottleneck}
\end{figure}

The empirical risk minimization function of $L_0$-SIGN in Equation \ref{fun:l0_sign_loss} can be approximately derived from Equation \ref{eq:ib} (Section \textit{A\iffalse\ref{appx:ib2sign}\fi} of Appendix gives a detailed derivation):
\begin{equation}
\small
 \min \mathcal{R}(\bm{\theta}, \bm{\omega}) \approx \min (I(\bm{X};\bm{S}) - \beta I(\bm{S};\bm{Y})). 
\end{equation}
Intuitively, the $L_0$ regularization in Equation \ref{fun:l0_sign_loss} minimizes a Kullback--Leibler divergence between every $e^{'}_{ij}$ and a Bernoulli distribution $Bern(0)$, and the $L_2$ regularization minimizes a Kullback–Leibler divergence between every $\bm{z}_{ij}$ and a multivariate standard distribution $\mathcal{N}(\bm{0}, \bm{I})$. As illustrated in Figure \ref{fig:bottleneck}, the statistical interaction analysis results, $\bm{s}$, can be approximated as the relevant part $\bm{S}$ in Equation \ref{eq:ib}.

Therefore, through training $L_0$-SIGN, $\bm{s}$ is the most compact (beneficial) representation of the interactions for recommendation prediciton. This provides a theoretical guarantee that the most beneficial feature interactions will be detected.

\subsubsection{Relation to Spike-and-slab Distribution}
The spike-and-slab distribution \cite{mitchell1988bayesian} is the golden standard in sparsity. It is defined as a mixture of a delta spike at zero and a continuous distribution over the real line (e.g., a standard normal):
\begin{equation}
\label{fun:ssd}
\small
\begin{gathered}
    p(a)= Bern(\pi),  \quad\quad  p(\theta\mid a=0)=\delta(\theta), \\
    p(\theta\mid a=1)=\mathcal{N}(\theta\mid 0,1).
\end{gathered}
\end{equation}
We can regard the spike-and-slab distribution as the product of a continuous distribution and a Bernoulli distribution. In $L_0$-SIGN, the predicted edge value vector $\bm{e}^{'}$ (the vector form of $E^{'}$) is a multivariate Bernoulli distribution and can be regarded as $p(a)$. The interaction modeling result $\bm{z}$ is a multivariate normal distribution and can be regarded as $p(\theta)$. Therefore, $L_0$-SIGN's statistical interaction analysis results, $\bm{s}$, is a multivariate spike-and-slab distribution that performs edge sparsification by discarding unbeneficial feature interactions. The retained edges in the spike-and-slab distribution are sufficient for $L_0$-SIGN to provide accurate predictions.

\subsubsection{Statistical Interaction in SIGN}
The feature interaction modeling in SIGN strictly follows the definition of statistical interaction, which is formally described in Theorem \ref{thm:interaction} (Section \textit{B\iffalse\ref{appx:prove_theorem}\fi} of Appendix gives the proof):

\begin{theorem}
\label{thm:interaction}
(\textbf{Statistical Interaction in SIGN}) Consider a graph $G(X, E)$, where $X$ is the node set and $E=\{e_{ij}\}_{ i,j \in X}$ is the edge set that $e_{ij}\in\{0,1\}, e_{ij}=e_{ji}$. Let $G(X, E)$ be the input of SIGN function $f_{S}(G)$ in Equation \ref{fun:si_final_function}, then the function flags pairwise statistical interaction between node $i$ and node $j$ if and only if they are linked by an edge in $G(X, E)$, i.e., $e_{ij}=1$.
\end{theorem}

Theorem \ref{thm:interaction} guarantees that SIGN will capture the interaction information from node pairs \textit{iff} they are linked by an edge. This ensures SIGN to accurately leverage the detected beneficial feature interactions for both inferring the target outcome and meanwhile providing useful feedback to the $L_0$ edge prediction component for better detection.

\subsubsection{Statistical Interaction in $L_0$-SIGN}
$L_0$-SIGN provides the same feature interaction modeling ability as SIGN, since we can simply extend Theorem \ref{thm:interaction} to Corollary \ref{coro:interaction} (Section \textit{C\iffalse\ref{appx:corollary}\fi} of Appendix gives the proof):

\begin{corollary}
\label{coro:interaction}
(\textbf{Statistical Interaction in $L_0$-SIGN}) Consider a graph $G$ that the edge set is unknown. Let $G$ be the input of $L_0$-SIGN function $F_{LS}(G)$ in Equation \ref{fun:l0_final_function}, the function flags pairwise statistical interaction between node $i$ and node $j$ if and only if they are predicted to be linked by an edge in $G$ by $L_0$-SIGN, i.e., $e^{'}_{ij}=1$.
\end{corollary}

\section{Experiments}
\label{sec:experiences}

We focuses on answering three questions: (i) how $L_0$-SIGN performs compared to baselines and whether SIGN helps detect more beneficial interactions for better performance?
(ii) How is the detection ability of $L_0$-SIGN?
(iii) Can the statistical interaction analysis results provide potential explanations for the recommendations predictions?

\subsection{Experimental Protocol}

\subsubsection{Datasets}

We study two real-world datasets for recommender systems to evaluate our model:

\noindent\textbf{Frappe} \cite{baltrunas2015frappe} is a context-aware recommendation dataset that records app usage logs from different users with eight types of contexts (e,g, weather). Each log is a graph (without edges), and nodes are user ID, app ID, or the contexts.

\noindent\textbf{MovieLens-tag} \cite{he2017neural} focuses on the movie tag recommendation (e.g., ``must-see"). Each data instance is a graph, with nodes as user ID, movie ID, and a tag that the user gives to the movie. 

To evaluate the question (ii), we further study two datasets for graph classification, which will be discussed later. The statistics of the datasets are summarized in Table \ref{tab:dataset}.

\begin{table}[t]
\centering
\small
\begin{sc}
\begin{tabular}{>{\arraybackslash}p{1.4cm}>{\centering\arraybackslash}p{1.5cm}>{\centering\arraybackslash}p{1.5cm}>{\centering\arraybackslash}p{2.3cm}}
\hline
\textbf{Dataset} & \textbf{\#Features} & \textbf{\#Graphs} & \textbf{\#Nodes/Graph} \\
\hline
Frappe     & 5,382 & 288,609  & 10  \\  
MovieLens  & 90,445 & 2,006,859 & 3  \\ 
Twitter    & 1,323 & 144,033 & 4.03 \\  
DBLP       & 41,324  & 19,456  & 10.48\\ 
\hline
\end{tabular}
\end{sc}
\caption{Dataset statistics. All datasets are denoted in graph form. Twitter and DBLP datasets are used for question (ii).}
\label{tab:dataset}
\end{table}

\subsubsection{Baselines}

We compare our model with recommender system baselines that model all feature interactions:

\noindent \textbf{FM} \cite{koren2008factorization}: It is one of the most popular recommendation algorithms that models every feature interactions.
\textbf{AFM} \cite{xiao2017attentional}: Addition to FM, it calculates an attention value for each feature interaction.
\textbf{NFM} \cite{he2017neural}: It replaces the dot product procedure of FM by an MLP.
\textbf{DeepFM} \cite{guo2017deepfm}: It uses MLP and FM for interaction analysis, respectively. 
\textbf{xDeepFM} \cite{lian2018xdeepfm}: It is an extension of DeepFM that models feature interactions in both explicit and implicit way.
\textbf{AutoInt} \cite{song2019autoint}: It explicitly models all feature interactions using a multi-head self-attentive neural network.
We use the same MLP settings in all baselines (if use) as our interaction modeling function $h$ in SIGN for fair comparison.

\subsubsection{Experimental Set-up}

In the experiments, we use element-wise mean as both linear aggregation functions $\psi(\cdot)$ and $\phi(\cdot)$. The linear function $g(\cdot)$ is a weighted sum function (i.e., $g(\bm{u_{i}^{'}})= \bm{w}_{g}^{T}\bm{u_{i}^{'}}$, where $\bm{w}_{g}\in\mathbb{R}^{d\times 1}$ are the weight parameters). 
For the interaction modeling function $h(\cdot)$, we use a MLP with one hidden layer after element-wise product: $h(\bm{u}_i, \bm{u}_j)=\bm{W}^{h}_2\sigma(\bm{W}^{h}_1(\bm{u}_i\odot\bm{u}_j) + \bm{b}^{h}_1) + \bm{b}^{h}_2$, where $\bm{W}^{h}_1, \bm{W}^{h}_2, \bm{b}^{h}_1, \bm{b}^{h}_2$ are parameters of MLP and $\sigma(\cdot)$ is a Relu activation function.
We implement the edge prediction model based on the neural collaborative filtering framework \cite{he2017neuralCF}, which has a similar form to $h(\cdot)$: $f_{ep}(\bm{v}^{e}_{i}, \bm{v}^{e}_{j})=\bm{W}^{e}_2\sigma(\bm{W}^{e}_1(\bm{v}^{e}_i\odot\bm{v}^{e}_j) + \bm{b}^{e}_1) + \bm{b}^{e}_2$.
We set node embedding sizes for both interaction modeling and edge prediction to 8 (i.e., $b,d=8$) and the sizes of hidden layer for both $h$ and $f_{ep}$ to 32. We choose the weighting factors $\lambda_1$ and $\lambda_2$ from $[1\times 10^{-5}, 1\times 10^{-1}]$ that produce the best performance in each dataset\footnote{Our implementation of our $L_0$-SIGN model is available at \href{https://github.com/ruizhang-ai/SIGN-Detecting-Beneficial-Feature-Interactions-for-Recommender-Systems}{https://github.com/ruizhang-ai/SIGN-Detecting-Beneficial-Feature-Interactions-for-Recommender-Systems}.}. 

Each dataset is randomly split into training, validation, and test datasets with a proportion of 70\%, 15\%, and 15\%. We choose the model parameters that produce the best results in validation set when the number of predicted edges being steady. We use accuracy (ACC) and the area under a curve with Riemann sums (AUC) as evaluation metrics.

\subsection{Model Performance}

Table \ref{tab:performance_recom} shows the results of comparing our model with recommendation baselines, with the best results for each dataset in bold. The results of SIGN are using all feature interactions as input, i.e., the input is a complete graph. 

Through the table, we observe that:
(i) $L_0$-SIGN outperforms all baselines. It shows $L_0$-SIGN's ability in providing accurate recommendations. Meanwhile, SIGN solely gains comparable results to competitive baselines, which shows the effectiveness of SIGN in modeling feature interactions for recommendation.
(ii) $L_0$-SIGN gains significant improvement from SIGN. It shows that more accurate predictions can be delivered by retaining only beneficial feature interactions and effectively modeling them. 
(iii) FM and AFM (purely based on dot product to model interactions) gain lower accuracy than other models, which shows the necessity of using sophisticated methods (e.g., MLP) to model feature interactions for better predictions.
(iv) The models that explicitly model feature interactions (xDeepFM, AutoInt, and $L_0$-SIGN) outperform those that implicitly model feature interactions (NFM, DeepFM). It shows that explicit feature interaction analysis is promising in delivering accurate predictions.

\begin{table}[t]
\centering
\small
\begin{sc}
\begin{tabular}{l|cc|cc}
\hline
 & \multicolumn{2}{c|}{\textbf{Frappe}} & \multicolumn{2}{c}{\textbf{MovieLens}} \\
 & auc & acc & auc & acc \\
\hline
FM      & 0.9263 & 0.8729 & 0.9190 & 0.8694\\
AFM     & 0.9361 & 0.8882 & 0.9205 & 0.8711 \\
NFM     & 0.9413 & 0.8928 & 0.9342 & 0.8903\\
DeepFM  & 0.9422 & 0.8931 & 0.9339 & 0.8895\\
xDeepFM & 0.9435 & 0.8950 & 0.9347 & 0.8906\\
AutoInt & 0.9432 & 0.8947 & 0.9351 & 0.8912\\
\hline
SIGN  & 0.9448 & 0.8974 & 0.9354 & 0.8921 \\
$L_0$-SIGN  & \textbf{0.9580} & \textbf{0.9174} & \textbf{0.9407} & \textbf{0.8970}\\
\hline
\end{tabular}
\end{sc}
\caption{Summary of results in comparison with baselines.}
\label{tab:performance_recom}
\end{table}

\subsection{Comparing SIGN with Other GNNs in Our Model}
To evaluate whether our $L_0$ edge prediction technique can be used on other GNNs and whether SIGN is more suitable than other GNNs in our model, we replace SIGN with existing GNNs: GCN \cite{kipf2016semi}, Chebyshev filter based GCN (Cheby) \cite{defferrard2016convolutional} and GIN \cite{xu2018powerful}.
We run on two datasets for graph classification since they contain heuristic edges (used to compare with the predicted edges):

\noindent\textbf{Twitter} \cite{pan2015cogboost} is extracted from twitter sentiment classification. Each tweet is a graph with nodes being word tokens and edges being the co-occurrence between two tokens in each tweet.

\noindent\textbf{DBLP} \cite{pan2013graph} consists of papers with labels indicating either they are from DBDM or CVPR field. Each paper is a graph with nodes being paper ID or keywords and edges being the citation relationship or keyword relations.

\begin{table}
\centering
\small
\begin{sc}
\begin{tabular}{l|cc|cc}
\hline
\multirow{2}{*}{} & \multicolumn{2}{c|}{\textbf{Twitter}} & \multicolumn{2}{c}{\textbf{DBLP}} \\
 & auc & acc & auc & acc \\
\hline
GCN  & 0.7049 & 0.6537 & 0.9719 & 0.9289 \\
$L_0$-GCN & 0.7053 & 0.6543 & 0.9731 & 0.9301  \\
\hline
CHEBY & 0.7076 & 0.6522 & 0.9717 & 0.9291 \\
$L_0$-CHEBY & 0.7079 & 0.6519 &  0.9719 & 0.9297 \\
\hline
GIN & 0.7149 & 0.6559 & 0.9764 & 0.9319 \\
$L_0$-GIN & 0.7159 & 0.6572 &  0.9787 & 0.9328 \\
\hline
SIGN & 0.7201 & 0.6615  & 0.9761 & 0.9316 \\
$L_0$-SIGN & \textbf{0.7231} & \textbf{0.6670} & \textbf{0.9836} & \textbf{0.9427} \\
\hline
\end{tabular}
\end{sc}
\caption{The results in comparison with existing GNNs. The model names without ``$L_0$-" use heuristic edges, and those with ``$L_0$-" automatically detect edges via our $L_0$ edge prediction technique.}
\label{tab:performance_graph}
\end{table}
\begin{figure}
\begin{tikzpicture}
  \centering
  \begin{small}
  \begin{axis}[
        ybar, axis on top,
        height=3.3cm, width=0.49\textwidth,
        bar width=0.23cm,
        major grid style={draw=white},
        ymin=-0.2, ymax=1,
        axis x line*=bottom,
        y axis line style={opacity=0},
        ylabel shift = -4 pt,
        tickwidth=0pt,
        enlarge x limits=true,
        legend image post style={scale=1.0},
        legend style={
            draw=none,
            at={(0.31,1.1)},
            anchor=north,
            legend columns=2,
            nodes={scale=0.9},
            /tikz/every even column/.append style={column sep=0.5cm}
        },
        ylabel={\textit{Imp.} (\%)},
        symbolic x coords={
           GCN, CHEBY, GIN, SIGN},
       xtick=data,
    ]
    \addplot [pattern = crosshatch dots] coordinates {
      (GCN, 0.057)
      (CHEBY, 0.042) 
      (GIN, 0.140)
      (SIGN, 0.417) 
      };
   \addplot [pattern=horizontal lines] coordinates {
      (GCN, 0.091)
      (CHEBY, -0.046) 
      (GIN, 0.198)
      (SIGN, 0.831)  };
   \addplot [pattern = grid, fill=gray!80] coordinates {
      (GCN, 0.123)
      (CHEBY, 0.021) 
      (GIN, 0.236)
      (SIGN, 0.737) };
    \addplot [pattern=crosshatch, fill=black!80] coordinates {
      (GCN, 0.129)
      (CHEBY, 0.064) 
      (GIN, 0.097)
      (SIGN, 0.996) };
    \legend{Twitter AUC,Twitter ACC, DBLP AUC, DBLP ACC}
  \end{axis}
  \end{small}
  \end{tikzpicture}
  \caption{The comparison of improvement via the $L_0$ edge prediction technique to using heuristic edges in Table \ref{tab:performance_graph}.}
  \label{fig:gnn_improvement}
\end{figure}
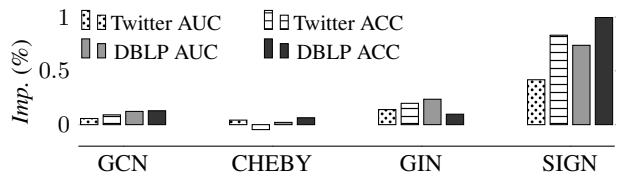

Table \ref{tab:performance_graph} shows the accuracies of each GNN that runs both on using the given (heuristic) edges (without ``$L_0$" in name) and on predicting edges with our $L_0$ edge prediction model (with ``$L_0$" in name). We can see that for all GNNs, $L_0$-GNNs gain competitive results comparing to corresponding GNNs. It shows that our model framework lifts the requirement of domain knowledge in defining edges in order to use GNNs in some situations. 
Also, $L_0$-SIGN outperforms other GNNs and $L_0$-GNNs, which shows $L_0$-SIGN's ability in detecting beneficial feature interactions and leveraging them to perform accurate predictions. 
Figure \ref{fig:gnn_improvement} shows each GNN's improvement from the results from ``GNN" to ``$L_0$-GNN" in Table \ref{tab:performance_graph}. It shows that among all the different GNN based models, SIGN gains the largest improvement from the $L_0$ edge prediction technique v.s. SIGN without $L_0$ edge prediction. It shows that SIGN can help the $L_0$ edge prediction model to better detect beneficial feature interactions for more accurate predictions.

\subsection{Evaluation of Interaction Detection}
\label{subsec:inter_detec_eval}
We then evaluate the effectiveness of beneficial feature interaction detection in $L_0$-SIGN.

\subsubsection{Prediction Accuracy vs. Numbers of Edges}
Figure \ref{fig:AUC_vs_edge_epoch} shows the changes in prediction accuracy and the number of edges included while training. 
The accuracy first increases while the number of included edges decreases dramatically. Then, the number of edges becomes steady, and the accuracy reaches a peak at similar epochs. It shows that our model can recognize unbeneficial feature interactions and remove them for better prediction accuracy.

\begin{figure}[ht]
\centering
\begin{tikzpicture}
\begin{axis}[mystyle2, title =Frappe, xlabel={Epochs}, xtick={1, 50, 100, 150, 200}, xmin=0, xmax=200, ymin=0.6, ymax=1.03, legend style={legend style={draw=none, at={(0.98,0.43)}},anchor=east, nodes={scale=0.7, transform shape}}, legend image post style={scale=0.6}]
    \addplot[mark=x] coordinates{
         (1,    0.9187)
    	(25,   0.9486)
    	(50,   0.9583)
    	(75,   0.9580)
    	(100,   0.9579)
    	(125,   0.9578)
    	(150,   0.9577)
    	(175,   0.9576)
    	(200,   0.9575)
    };
    \addplot[mark=square*] coordinates {
    	(1,    0.8858219)
    	(25,   0.9288478)
    	(50,   0.92854754)
    	(75,   0.92990583)
    	(100,   0.92959373)
    	(125,   0.92938583)
    	(150,   0.9294782)
    	(175,   0.92936277)
    	(200,   0.92920106)
    };
    \addplot[style={dashed},mark=*] coordinates{
    	(1,    0.9990)
    	(25,   0.9818)
    	(50,   0.6760)
    	(75,   0.6708)
    	(100,   0.6699)
    	(125,   0.66982)
    	(150,   0.6692)
    	(175,   0.6690)
    	(200,   0.6689)
    };
    \legend{AUC, ACC, Edges (\%)}
\end{axis}
\end{tikzpicture}
\begin{tikzpicture}
\begin{axis}[mystyle2, title =Movielens, xlabel={Epochs}, xtick={1,100,200,300}, xmin=0, xmax=300, ymin=0.55, ymax=1.03, legend style={draw=none, at={(0.98,0.39)},anchor=east, nodes={scale=0.7, transform shape}}, legend image post style={scale=0.6}]
\addplot[mark=x] coordinates {
	(1,    0.9099)
	(50,   0.9374)
	(100,   0.9399)
	(150,  0.9398)
	(200,  0.9407)
	(250,  0.9397)
	(300,  0.9395)
};
\addplot[mark=square*] coordinates {
	(1,    0.8490696)
	(50,   0.89700883)
	(100,   0.8967365)
	(150,  0.89678907)
	(200,  0.89443364)
	(250,  0.89496193)
	(300,  0.89406296)
};
\addplot[style={dashed}, mark=*] coordinates{
	 (1,    1.0)
	(50,   0.97)
	(100,  0.63)
	(150,  0.5961)
	(200,  0.5947)
	(250,  0.59407)
	(300, 0.5935)
};\legend{AUC, ACC, Edges (\%)}
\end{axis}
\end{tikzpicture}
\begin{tikzpicture}
\begin{axis}[mystyle2, title =Twitter, xlabel={Epochs}, xtick={1,50,100,150,200}, xmin=0, xmax=200, ymin=0.55, ymax=1.03, legend style={draw=none, at={(0.98,0.75)},anchor=east, nodes={scale=0.7, transform shape}}, legend image post style={scale=0.6}]
\addplot[mark=x] coordinates {
	(1,    0.69488)
	(25,   0.72143)
	(50,  0.72351)
	(75,  0.72337)
	(100,  0.72324)
	(125,  0.72311)
	(150,  0.72300)
	(175,  0.72291)
	(200,  0.72284)
};
\addplot[mark=square*] coordinates {
    (1,    0.639477)
	(25,   0.6651604)
	(50,  0.6670215)
	(75,  0.66660494)
	(100,  0.6657255)
	(125,  0.66498494)
	(150,  0.6652164)
	(175,  0.66475356)
	(200,  0.6651238)
};
\addplot[style={dashed}, mark=*] coordinates{
	 (1,    1.0)
	(25,   0.9518)
	(50,  0.7479)
	(75,  0.7464)
	(100,  0.7453)
	(125,  0.7446)
	(150,  0.7443)
	(175,  0.7440)
	(200,  0.7437)
};\legend{AUC, ACC, Edges (\%)}
\end{axis}
\end{tikzpicture}
\begin{tikzpicture}
\begin{axis}[mystyle2, title =DBLP, xlabel={Epochs}, xtick={1,50,100,150,200,250}, xmin=0, xmax=250, ymin=0.55, ymax=1.03, legend style={draw=none, at={(0.52,0.29)},anchor=east, nodes={scale=0.7, transform shape}}, legend image post style={scale=0.6}]
\addplot[mark=x] coordinates {
    (1,    0.9565)
	(25,   0.9808)
	(50,   0.98301)
	(75,   0.9836)
	(100,   0.9827)
	(125,   0.9824)
	(150,   0.98123)
	(175,   0.98109)
	(200,   0.98100)
	(225,   0.98093)
	(250,   0.98082)
};
\addplot[mark=square*] coordinates {
	(1,    0.911826)
	(25,   0.93427954)
	(50,   0.9369647)
	(75,   0.94274605)
	(100,   0.9421963)
	(125,   0.9422518)
	(150,   0.9405944)
	(175,   00.9402518)
	(200,   0.9395944)
	(225,   0.93727954)
	(250,   0.936937)
};
\addplot[style={dashed}, mark=*] coordinates{
	(1,    0.9913)
	(25,   0.9990)
	(50,   0.9963)
	(75,   0.9818)
	(100,   0.9124)
	(125,   0.8002)
	(150,   0.7522)
	(175,   0.7314)
	(200,   0.7203)
	(225,   0.7131)
	(250,   0.7084)
};\legend{AUC, ACC, Edges (\%)}
\end{axis}
\end{tikzpicture}
\caption{The changes in prediction accuracy and number of edges (in percentage) while training.}
\label{fig:AUC_vs_edge_epoch}
\end{figure}
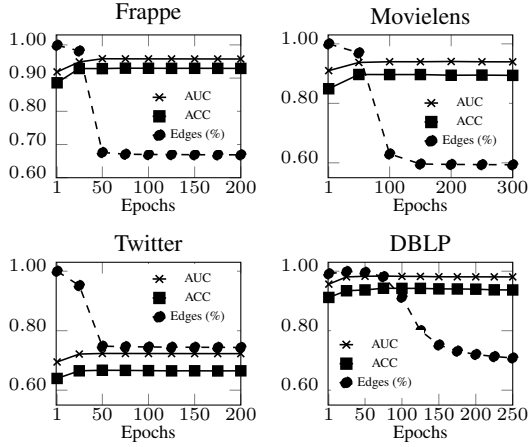

\subsubsection{Predicted Edges vs. Reversed Edges} 

We evaluate how \textit{predicted edges} and \textit{reversed edges} (the excluded edges) influence the performance. We generate 5 edge sets with a different number of edges by randomly selecting from 20\% predicted edges (ratio 0.2) to 100\% predicted edges (ratio 1.0). We generate another 5 edge sets similarly from reversed edges. Then, we run SIGN on each edge set for 5 times, and show the averaged results in Figure \ref{fig:auc_in_edges}. 
It shows that using predicted edges always gets higher accuracy than using reversed edges. The accuracy stops increasing when using reversed edges since the ratio 0.6, while it continually improves when using more predicted edges. According to Definition \ref{def:rfi}, the detected edges are proved beneficial since considering more of them can provide further performance gain and the performance is the best when using all predicted edges, while considering the reversed edges cannot. Note that the increment of reversed edges from 0.2 to 0.6 may come from covering more nodes since features solely can provide some useful information for prediction.
\begin{figure}[t]
\centering
        \begin{tikzpicture}
        \begin{axis}[mystyle2, title=Frappe, xlabel={Ratio}, xtick = {0.2,0.4,...,1.0},legend style={draw=none, at={(0.99,0.17)},anchor=east, nodes={scale=0.68, transform shape}}, legend image post style={scale=0.5}, legend columns=2, xmin=0.2,xmax=1.0, ymin=0.79]
        \addplot[mark=diamond*] coordinates {
            (0.2,   0.9115)
        	(0.4,   0.9449)
        	(0.6,   0.9502)
        	(0.8,  0.9559)
        	(1.0,  0.95803)
        };
        \addplot[color=gray, mark=triangle*] coordinates{
            (0.2,   0.9103)
        	(0.4,   0.9337)
        	(0.6,  0.9336)
        	(0.8,  0.9350)
        	(1.0,  0.9354)
        };
         \addplot[mark=*] coordinates {
            (0.2,   0.8461)
            (0.4,   0.8916)
            (0.6,   0.9165)
            (0.8,   0.9276)
            (1.0,   0.9299)
        };
        \addplot[color=gray, mark=x] coordinates{
            (0.2,   0.8432)
            (0.4,   0.8723)
            (0.6,   0.8713)
            (0.8,  0.8718)
            (1.0,  0.8715)
        };
        \legend{Pred (AUC), Rev (AUC), Pred (ACC), Rev (ACC)}
        \end{axis}
        \end{tikzpicture}
        \begin{tikzpicture}
        \begin{axis}[mystyle2, title=Movielens, xlabel={Ratio}, xtick = {0.2,0.4,...,1.0},legend style={draw=none, at={(0.99,0.17)},anchor=east, nodes={scale=0.68, transform shape}}, legend image post style={scale=0.5},legend columns=2, xmin=0.2,xmax=1.0, ymin=0.78]
        \addplot[mark=diamond*] coordinates {
            (0.2,   0.8932)
        	(0.4,   0.9203)
        	(0.6,  0.9369)
        	(0.8,  0.9391)
        	(1.0,  0.9407)
        };
        \addplot[color=gray, mark=triangle*] coordinates{
            (0.2,   0.8903)
        	(0.4,   0.9034)
        	(0.6,  0.9071)
        	(0.8,  0.9072)
        	(1.0,  0.90729)
        };
        \addplot[mark=*] coordinates {
            (0.2,   0.8287)
            (0.4,   0.8578)
            (0.6,  0.8824)
            (0.8,  0.8945)
            (1.0,  0.8972)
        };
        \addplot[color=gray, mark=x] coordinates{
            (0.2,   0.8219)
            (0.4,   0.8508)
            (0.6,  0.8582)
            (0.8,  0.8583)
            (1.0,  0.8583)
        };
        \legend{Pred (AUC), Rev (AUC), Pred (ACC), Rev (ACC)}
        \end{axis}
        \end{tikzpicture}
        \begin{tikzpicture}
        \begin{axis}[mystyle2, title=Twitter, xlabel={Ratio}, xtick = {0.2,0.4,...,1.0},legend style={draw=none, at={(0.99,0.17)},anchor=east, nodes={scale=0.68, transform shape}}, legend image post style={scale=0.5},legend columns=2, xmin=0.2,xmax=1.0, ymin=0.59]
        \addplot[mark=diamond*] coordinates {
            (0.2,   0.686)
            (0.4,   0.709)
            (0.6,   0.7184)
            (0.8,   0.7228)
            (1.0,   0.7231)
        };
        \addplot[color=gray, mark=triangle*] coordinates{
            (0.2,   0.683)
            (0.4,   0.697)
            (0.6,   0.704)
            (0.8,  0.704)
            (1.0,  0.704)
        };
        \addplot[mark=*] coordinates {
            (0.2,   0.6376)
            (0.4,   0.6488)
            (0.6,   0.6571)
            (0.8,   0.6645)
            (1.0,   0.6670)
        };
        \addplot[color=gray, mark=x] coordinates{
            (0.2,   0.6343)
            (0.4,   0.6409)
            (0.6,   0.6416)
            (0.8,  0.642)
            (1.0,  0.6418)
        };
        \legend{Pred (AUC), Rev (AUC), Pred (ACC), Rev (ACC)}
        \end{axis}
        \end{tikzpicture}
        \begin{tikzpicture}
        \begin{axis}[mystyle2, title=DBLP, xlabel={Ratio}, xtick = {0.2,0.4,...,1.0},legend style={draw=none, at={(0.99,0.17)},anchor=east, nodes={scale=0.68, transform shape}}, legend image post style={scale=0.5},legend columns=2, xmin=0.2,xmax=1.0, ymin=0.885]
        \addplot[mark=diamond*] coordinates {
            (0.2,   0.9657)
            (0.4,   0.9771)
            (0.6,  0.9808)
            (0.8,  0.9827)
            (1.0,  0.9836)
        };
        \addplot[color=gray, mark=triangle*] coordinates{
            (0.2,   0.9639)
            (0.4,   0.9735)
            (0.6,  0.9767)
            (0.8,  0.9768)
            (1.0,  0.9768)
        };
        \addplot[mark=*] coordinates {
            (0.2,   0.9165)
            (0.4,   0.9353)
            (0.6,  0.9381)
            (0.8,  0.9406)
            (1.0,  0.9427)
        };
        \addplot[color=gray, mark=x] coordinates{
            (0.2,   0.9117)
            (0.4,   0.9311)
            (0.6,  0.9361)
            (0.8,  0.9357)
            (1.0,  0.9359)
        };
        \legend{Pred (AUC), Rev (AUC), Pred (ACC), Rev (ACC)}
        \end{axis}
        \end{tikzpicture}
    \caption{Evaluating different number of edges. ``Pred" is the predicted edges and ``Rev" is the reversed edges.}
    \label{fig:auc_in_edges}
\end{figure}
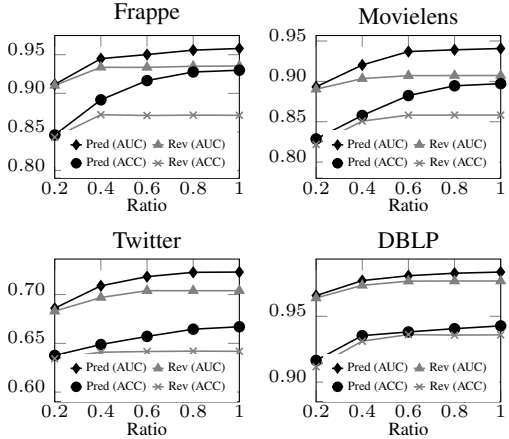

\begin{figure}
    \centering
    \includegraphics[width=0.95\columnwidth]{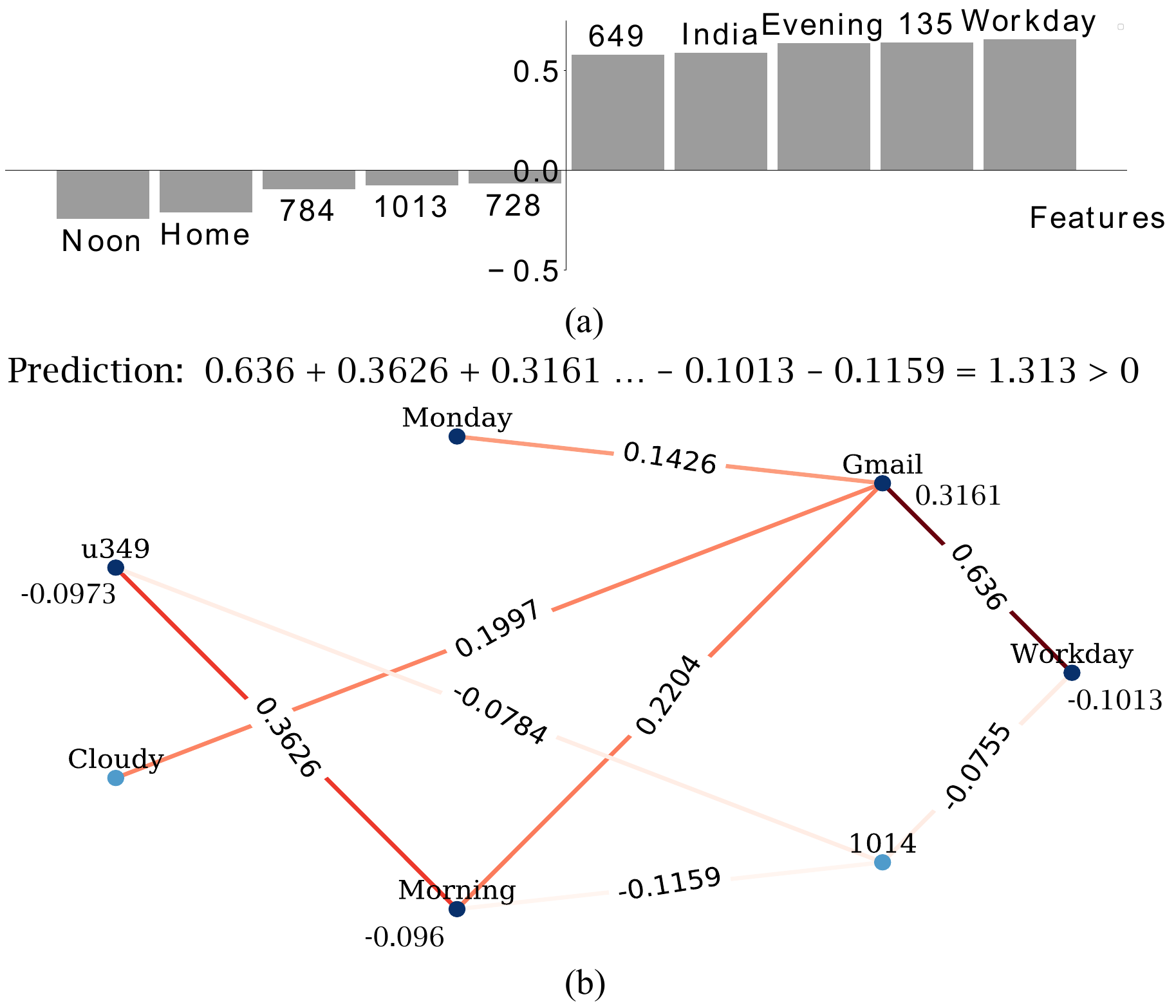}
    \caption{(a) The highest interaction values with \textit{Gmail}. The numbers are city indexes (e.g., 1014). (b) The prediction that the user \textit{u349} will use \textit{Gmail} as the prediction value is 1.313 $>$ 0. Darker edges are higher interaction values.}
    \label{fig:case_study}
\end{figure}

\subsection{Case Study}

Another advantage of interaction detection is to automatically discover potential explanations about recommendation predictions. 
We conduct case studies on the Frappe dataset to show how $L_0$-SIGN provides the explanations.

We first show the beneficial interactions that have the highest interaction values with \textit{Gmail} in Figure \ref{fig:case_study}a. We can see that \textit{Workday} has a high positive value, while \textit{Home} has a high negative value. It may show that Gmail is very likely to be used in workdays since people need to use Gmail while working, whereas Gmail is not likely to be used at home since people may not want to dealing with email while resting.
We then show how $L_0$-SIGN provides potential explanations for the prediction on each data sample.
Figure \ref{fig:case_study}b visualizes a prediction result that a user (\textit{u349}) may use \textit{Gmail}. The graph shows useful information for explanations. Despite the beneficial interactions such as $<\textit{Gmail}, \textit{Workday}>$ that have discussed above, we can also see that \textit{Cloudy} and \textit{Morning} have no beneficial interaction, meaning that whether it is a cloudy morning does not contribute to decide the user's will of using Gmail.

\section{Conclusion and Future Work}
\label{sec:conclusion}
We are the first to propose and formulate the problem of detecting beneficial feature interactions for recommender systems.
We propose $L_0$-SIGN to detect the beneficial feature interactions via a graph neural network approach and $L_0$ regularization.
Theoretical analyses and extensive experiments show the ability of $L_0$-SIGN in detecting and modeling beneficial feature interactions for accurate recommendations.
In future work, we will model high-order feature interactions in GNNs with theoretical foundations to effectively capture high-order feature interaction information in graph form.

\section*{Acknowledgments}
This work is supported by the China Scholarship Council.

\bibliography{ref}
\bibstyle{aaai21}

\newpage
\quad
\newpage
\appendix

\section{Derivation from IB to $L_0$-SIGN}
\label{appx:ib2sign}

Recall that the empirical risk minimization procedure of $L_0$-SIGN in Equation \ref{fun:l0_sign_loss} is:
\begin{equation*}
\begin{split}
    \mathcal{R}(\bm{\theta}, \bm{\omega})=&\frac{1}{N}\sum_{n=1}^{N}(\mathcal{L}(F_{LS}(G_n;\bm{\omega},\bm{\theta}),y_{n}) \\
    &+ \lambda_1 \sum_{i,j\in X_n}(\pi_{n})_{ij} + \lambda_2 \norm{\bm{z}_{n}}_2), \\
\end{split}
\end{equation*}

Deep variational information bottleneck method \cite{alemi2016deep} performs a variational approximation to the Information Bottleneck principle (Equation \ref{eq:ib}). Specifically, the function can be approximated by maximizing a lower bound $L$:
\begin{equation}
\label{eq_ib_lower_bound}
\begin{split}
    L \approx &\frac{1}{N} \sum_{n=1}^{N}[\int \mathop{d\tilde{s}_n} p(\tilde{s}_n\mid \tilde{x}_{n})\log q(\tilde{y}_n\mid \tilde{s}_n) \\
    &- \beta p(\tilde{s}_n\mid \tilde{x}_n)\log \frac{p(\tilde{s}_n\mid \tilde{x}_n)}{r(\tilde{s}_n)}],
\end{split}
\end{equation}
where $\tilde{x}_n,\tilde{y}_n,\tilde{s}_n$ are input, output and the middle state respectively, $r(\tilde{s}_n)$ is the variational approximation to the marginal $p(\tilde{s}_n)$.

Then, maximizing the lower bound $L$ equals to minimizing the $J_{IB}$:
\begin{equation}
\begin{split}
\label{eq_ib_J}
J_{IB} = \frac{1}{N} &\sum_{n=1}^{N}(\mathbb{E}_{\tilde{s}_n\sim p(\tilde{s}_n\mid \tilde{x}_n)} [-\log q(\tilde{y}_n\mid \tilde{s}_n)] \\
& + \beta KL [p(\tilde{s}_n\mid \tilde{x}_n), {r(\tilde{s}_n)}]),
\end{split}
\end{equation}
where $KL [p(\tilde{s}_n\mid \tilde{x}_n), {r(\tilde{s}_n)}]$ is the Kullback--Leibler divergence between $p(\tilde{s}_n\mid \tilde{x}_n)$ and $r(\tilde{s}_n)$.

In $L_0$-SIGN, the middle state part between input and output is the statistical interaction analysis result $\bm{s}_{n}$, which is the multiplication of predicted edge values $\bm{e}^{'}_{n}$ (the vector form of $E^{'}_{n}$) and interaction modeling results $\bm{z}_n$. $(e^{'}_{n})_{ij}=Bern((\pi_n)_{ij})$ so that $\bm{e}^{'}_{n}$ is a multivariate Bernoulli distribution, denoted as $p(\bm{e}^{'}_n\mid X_n)$. Similarly, $(\bm{z}_{n})_{ij}$ is a multivariate normal distribution $\mathcal{N}((\bm{z}_n)_{ij}, \Sigma_{ij})$ so that $\bm{z}_n$ is a multivariate normal distribution, denoted as $p(\bm{z}_n\mid X_n)$. Therefore, the distribution of $\bm{s}_n$ (denoted as $p(\bm{s}_n\mid X_n)$) is represented as:
\begin{equation}
\label{eq_ib_SIGN_pairwise}
\begin{split}
    p(\bm{s}_n\mid X_n) &= p(\bm{e}^{'}_n\mid X_n) p(\bm{z}_n\mid X_n) \\
    &= \mathbin\Vert_{i,j\in X_n} [Bern(\pi_{ij})\mathcal{N}((\bm{z}_n)_{ij}, \Sigma_{ij})],
\end{split}
\end{equation}
where $\Sigma_{ij}$ is a covariance matrix and $\mathbin\Vert$ is concatenation. 

Meanwhile, we set the variational approximation of the $\bm{s}_n$ being a concatenated multiplication of normal distributions with mean of $0$ and variance of $1$, and Bernoulli distributions that the probability of being $0$ is $1$. Then, the variational approximation in the vector form is:
\begin{equation}
\label{eq_ib_variational_apprx}
r(\bm{s}_n) = Bern(\bm{0}) \mathcal{N}(\bm{0}, \bm{I}),
\end{equation}
where $I$ is an identity matrix.

Combining Equation \ref{eq_ib_SIGN_pairwise} and Equation \ref{eq_ib_variational_apprx} into Equation \ref{eq_ib_J}, the minimization function correlating to $L_0$-SIGN becomes:
\begin{equation}
\label{eq_ib_merge1}
\begin{split}
    J_{IB} = \frac{1}{N} \sum_{n=1}^{N}&(\mathbb{E}_{\bm{s}_n\sim p(\bm{s}_n\mid X_n)} [-\log q(y_n\mid \bm{s}_n)] \\
    & + \beta KL[p(\bm{s}_n\mid X_n), r(\bm{s}_n)]).
\end{split}
\end{equation}

Next, we use the forward Kullback–Leibler divergence $KL[r(\bm{s}_n), p(\bm{s}_n\mid X_n)]$ to approximate the reverse Kullback–Leibler divergence in Equation \ref{eq_ib_merge1} to ensure the Kullback–Leibler divergence can be properly derived into the $L_0$ and $L_2$ activation regularization (will be illustrated in Equation \ref{eq_ib_kl_l0_approx} and Equation \ref{eq_ib_kl_l2}). We can perform this approximation because when the variational approximation $r(\tilde{z}_n)$ only contains one mode (e.g., Bernoulli distribution, normal distribution), both forward and reverse Kullback–Leibler divergence force $p(\tilde{z}_n\mid \tilde{x}_n)$ to cover the only mode and will have the same effect \cite{mackay2003information}. Then Equation \ref{eq_ib_merge1} becomes:
\begin{equation}
\begin{split}
\label{eq_ib_merge2}
& J_{IB}   \\
\approx & \frac{1}{N} \sum_{n=1}^{N}(\mathbb{E}_{\bm{s}_n\sim p(\bm{s}_n\mid X_n)} [-\log q(y_n\mid \bm{s}_n)] \\
&+ \beta KL [r(\bm{s}_n), p(\bm{s}_n\mid X_n)]) \\
= &\frac{1}{N} \sum_{n=1}^{N}(\mathbb{E}_{\bm{s}_n\sim p(\bm{s}_n\mid X_n)} [-\log q(y_n\mid \bm{s}_n)] \\
&+ \beta KL [Bern(\bm{0}) \mathcal{N}(\bm{0}, \bm{I}) , p(\bm{e}^{'}_n\mid X_n) p(\bm{z}_n\mid X_n)]) \\
= & \frac{1}{N} \sum_{n=1}^{N}(\mathbb{E}_{\bm{s}_n\sim p(\bm{s}_n\mid X_n)} [-\log q(y_n\mid \bm{s}_n)] \\
& + \beta (dKL [Bern(\bm{0}), p(\bm{e}^{'}_n\mid X_n)] \\
& \quad\quad + KL[\mathcal{N}(\bm{0}, \bm{I}), p(\bm{z}_n\mid X_n)])),
\end{split}
\end{equation}
where $d$ is the dimention of each vector $(\bm{z}_{n})_{ij}$.

Minimizing $\mathbb{E}_{\bm{s}_n\sim p(\bm{s}_n\mid X_n)} [-\log q(y_n\mid \bm{s}_n)]$ in Equation \ref{eq_ib_merge2} is equivalent to minimizing $\mathcal{L}(f_{LS}(G_n; \bm{\omega}, \bm{\theta}),y_{n})$ in Equation \ref{fun:l0_sign_loss}: its the negative log likelihood of the prediction as the loss function of $L_0$-SIGN, with $p(\bm{s}_n\mid X_n)$ being the edge prediction procedure and interaction modeling procedure, and $q(y_n\mid \bm{s}_n)$ being the aggregation procedure from the statistical interaction analysis result to the outcome $y_n$.

For the part $KL [Bern(\bm{0}), p(\bm{e}^{'}_n\mid X_n)]$ in Equation \ref{eq_ib_merge2}, $p(\bm{e}^{'}_n\mid X_n)=Bern(\bm{\pi}_n)$ is a multivariate Bernoulli distribution, so the KL divergence is:
\begin{equation}
\label{eq_ib_kl_l0}
\begin{split}
& KL [Bern(\bm{0}), p(\bm{e}^{'}_n\mid X_n)] = KL [Bern(\bm{0}), Bern(\bm{\pi}_n)]\\
 = &\sum_{i,j\in X_n}( 0 \log \frac{0}{(\pi_n)_{ij}} + (1 - 0) \log \frac{1-0}{1-(\pi_n)_{ij}} )\\
 = &\sum_{i,j\in X_n} \log \frac{1}{1-(\pi_n)_{ij}}.
\end{split}
\end{equation}

\begin{figure}[t]
\begin{center}
\begin{tikzpicture} 
\begin{axis}[xmin=0, xmax=1, ymax=1.2, samples=500, legend style={font=\small, at={(0.04,0.85)}, anchor=west},xlabel={$\pi$},width=0.42\textwidth]
\addplot[color=red]{log10(1/(1-x))};
\addplot[color=green]{x};
\legend{Bernoulli KL Divergence, Linear Approximation ($\gamma=1$)}
\end{axis}
\end{tikzpicture}
\caption{Linear Approximation vs. Bernoulli KL Divergence on different $\pi$ values.}
\label{fig_l0_approx}
\end{center}
\end{figure}
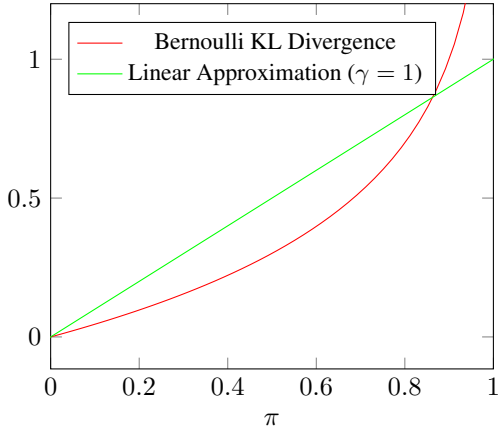

Next, we use a linear function $\gamma (\pi_n)_{ij}$ to approximate $\log \frac{1}{1-(\pi_n)_{ij}}$ in Equation \ref{eq_ib_kl_l0}, where $\gamma>0$ is a scalar constant. Figure \ref{fig_l0_approx} shows the values (penalization) of the Bernoulli KL divergence and its linear approximation on different $\pi$ values. It can be seen that both Bernoulli KL divergence and its approximations are monotonically increasing. In the empirical risk minimization procedure, they will have similar effects on penalizing those $\pi>0$. In addition, the approximation is more suitable for our model because: (i) it penalizes more than the Bernoulli KL divergence when $\pi$ is approaching $0$ (take more effort on removing unbeneficial feature interactions); and (ii) it gives reasonable (finite) penalization when $\pi$ is approaching $1$ (retrain beneficial feature interactions), while the Bernoulli KL divergence produces infinite penalization when $\pi=1$. 

Then the KL divergence of the multivariate Bernoulli distribution can be approximately calculated by:
\begin{equation}
\label{eq_ib_kl_l0_approx}
\small
\begin{split}
    KL [Bern(\bm{0}), p(\bm{e}^{'}_n\mid X_n)] &= \sum_{i,j\in X_n} \log \frac{1}{1-(\pi_n)_{ij}} \\
    &\approx \sum_{i,j\in X_n} \gamma (\pi_n)_{ij}.
\end{split}
\end{equation}

For the part $ KL[\mathcal{N}(\bm{0}, \bm{I}), p(\bm{z}_n\mid X_n)]$, the distribution $p(\bm{z}_n\mid X_n)$ is a multivariate normal distribution and is denoted as $\mathcal{N}(\bm{z}_n, \Sigma)$. If we assume all normal distributions in $p(\bm{z}_n\mid X_n)$ are i.i.d, and have the same variance (i.e., $\Sigma=diag(\sigma^{2}, \sigma^{2},\dots, \sigma^{2})$ where $\sigma$ is a constant), we can reformulate the KL divergence: 
\begin{equation}
\label{eq_ib_kl_l2}
\small
\begin{split}
& KL[\mathcal{N}(\bm{0}, \bm{I}), p(\bm{z}_n\mid X_n)] = KL[\mathcal{N}(\bm{0}, \bm{I}), \mathcal{N}(\bm{z}_n, \Sigma_n)] \\
 = &\frac{1}{2} (\Tr(\Sigma_n^{-1}\bm{I}) + (\bm{z}_n - \bm{0})^{T}\Sigma_{1}^{-1}(\bm{z}_n - \bm{0}) + \ln{\frac{\det\Sigma_n}{\det\bm{I}}} - d|X_n| ) \\
  = &\frac{1}{2} \sum_{i,j\in X_n}\sum_{k=1}^{d}(\frac{1}{\sigma^{2}} + \frac{(z_n)_{ijk}^{2}}{\sigma^{2}} + \ln{\sigma^{2}} - 1) \\
  = &\frac{1}{2\sigma^{2}} \sum_{i,j\in X_n}\sum_{k=1}^{d}((z_n)_{ijk}^{2} + C_2) \\
  \propto &\frac{1}{2\sigma^{2}} \sum_{i,j\in X_n}\sum_{k=1}^{d}(z_n)_{ijk}^{2},
\end{split}
\end{equation}
where $C_2=1 + \sigma^{2}\ln{\sigma^{2}} - \sigma^{2}$ is a constant and $(z_n)_{ijk}$ is the $k$th dimension of $(\bm{z}_n)_{ij}$. 

Relating to Equation \ref{fun:l0_sign_loss}, Equation \ref{eq_ib_kl_l0_approx} is exactly the $L_0$ activation regularization part and Equation \ref{eq_ib_kl_l2} is the $L_2$ activation regularization part in the empirical risk minimization procedure of $L_0$-SIGN, with $\lambda_1 = d\beta\gamma$ and $\lambda_2 = \frac{\beta}{2\sigma^{2}}$. Therefore, the empirical risk minimization procedure of $L_0$-SIGN is proved to be a variational approximation of minimizing the object function of IB ($J_{IB}$):
\begin{equation}
\label{eq_sign_is_ib}
 \min \mathcal{R}(\bm{\theta}, \bm{\omega}) \approx \min J_{IB}. 
\end{equation}

\section{Proof of Theorem \ref{thm:interaction}}
\label{appx:prove_theorem}

\textbf{Theorem 1.} \textit{(\textbf{Statistical Interaction in SIGN}) Consider a graph $G(X, E)$, where $X$ is the node set and $E=\{e_{ij}\}_{ i,j \in X}$ is the edge set that $e_{ij}\in\{0,1\}, e_{ij}=e_{ji}$. Let $G(X, E)$ be the input of SIGN function $f_{S}(G)$ in Equation \ref{fun:si_final_function}, then the function flags pairwise statistical interaction between node $i$ and node $j$ if and only if they are linked by an edge in $G(X, E)$, i.e., $e_{ij}=1$.}

\begin{proof}
We prove Theorem \ref{thm:interaction} by proving two lemmas:
\begin{lemma}
\label{lemma_inte2edge}
Under the condition of Theorem \ref{thm:interaction}, for a graph $G(X, E)$, if $f_{S}(G)$ shows pairwise statistical interaction between node $i$ and node $j$, where $i,j\in X$, then the two nodes are linked by an edge in $G(X, E)$, i.e., $e_{ij}=1$.
\end{lemma}

\begin{proof}
We prove this lemma by contradiction. Assume that the SIGN function $f_S$ with $G(X, E_{\char`\\e_{ij}})$ as input shows pairwise statistical interaction between node $i$ and node $j$, where $G(X, E_{\char`\\e_{ij}})$ is a graph with $E_{\char`\\e_{ij}}$ being a set of edges that $e_{ij}=0$.

Recall that the SIGN function in Equation \ref{fun:si_final_function}. Without losing generality, we set both the aggregation functions $\phi$ and $\psi$ being element-wise average. That is:
\begin{equation}
\label{eq_lemma_ignfun_average}
    f_{S}(G) =\frac{1}{|X|} \sum_{i\in X}(\frac{1}{\rho(i)}\sum_{j\in X}(e_{ij}h(\bm{u}_{i}, \bm{u}_{j}))),
\end{equation}
where $\rho(i)$ is the degree of node $i$.

From Equation \ref{eq_lemma_ignfun_average}, we know that the SIGN function can be regarded as the linear aggregation of non-additive statistical interaction modeling procedures $h(\bm{u}_{k}, \bm{u}_{m})$ for all nodes pairs ($k$, $m$) that $k,m\in X$ and $e_{km}=1$. Since $E_{\char`\\e_{ij}}$ does not contain an edge between $i$ and $j$ (i.e., $e_{ij}=0$), the SIGN function does not perform interaction modeling between the two nodes into final predictions. 

According to Definition \ref{def:spi}, since $i$ and $j$ have statistical interaction, we cannot find a replacement form of SIGN function like:
\begin{equation}
\label{eq_lemma_def1}
\begin{split}
    f_{S}(G) =& q_{\char`\\i}(\bm{u}_{1},...,\bm{u}_{i-1},\bm{u}_{i+1},...,\bm{u}_{|X|}) \\
    & + q_{\char`\\j}(\bm{u}_{1},...,\bm{u}_{j-1},\bm{u}_{j+1},...,\bm{u}_{|X|}),
\end{split}
\end{equation}
where $q_{\char`\\i}$ and $q_{\char`\\j}$ are functions without node $i$ and node $j$ as input, respectively.

However, from our assumption, since there is no edge between node $i$ and node $j$, there is no interaction modeling function that performs between them in $f_{S}(G)$. Therefore, we can easily find many such $q_{\char`\\i}$ and $q_{\char`\\j}$ that satisfy Equation \ref{eq_lemma_def1}. For example:
\begin{equation}
\label{eq_lemma_ignfun_average_fi}
\begin{split}
    q_{\char`\\i}&(\bm{u}_{1},...,\bm{u}_{i-1},\bm{u}_{i+1},...,\bm{u}_{|X|}) = \\
    & \frac{1}{|X|} \sum_{k \in  X\char`\\\{i\}}(\frac{1}{\rho(k)}\sum_{m\in X\char`\\\{i\}}(e_{km}h(\bm{u}_{k}, \bm{u}_{m}))),
\end{split}
\end{equation}

and
\begin{equation}
\label{eq_lemma_ignfun_average_fj}
\begin{split}
    q_{\char`\\j}(\bm{u}_{1},&...,\bm{u}_{j-1},\bm{u}_{j+1},...,\bm{u}_{|X|}) = \\
    &\frac{1}{|X|} \sum_{m\in X\char`\\\{j\}}(\frac{1}{\rho(i)}e_{im}h(\bm{u}_{i}, \bm{u}_{m})) \\ 
    & +\frac{1}{|X|}\sum_{k \in  X\char`\\\{j\}}(\frac{1}{\rho(k)}e_{ki}h(\bm{u}_{k}, \bm{u}_{i})).
\end{split}    
\end{equation}

Therefore, it contradicts our assumption. Lemma \ref{lemma_inte2edge} is proved. 
\end{proof}

\begin{lemma}
\label{lemma_edge2inte}
Under the condition of Theorem \ref{thm:interaction}, for a graph $G(X, E)$, if an edge links node $i$ and node $j$ in $G$ (i.e., $i,j\in X$ and $e_{ij}=1$), then $f_{S}(G)$ shows pairwise statistical interaction between node $i$ and node $j$.  
\end{lemma}

\begin{proof}
We prove this lemma by contradiction as well. Assume there is a graph $G(X, E)$ with a pair of nodes $(i,j)$ that $e_{ij} = 1$, but shows no pairwise statistical interaction between this node pair in $f_{S}(G)$. 

Since $e_{ij} =1$, we can rewrite SIGN function as:
\begin{equation}
\label{eq_lemma_haseij}
\begin{split}
    f_{S}(G) =&\frac{1}{|X|} \sum_{k \in  X}(\frac{1}{\rho(k)}\sum_{m\in X}(e_{km}h(\bm{u}_{k}, \bm{u}_{m}))) \\
    &+ \frac{\rho(i) + \rho(j)}{|X| \rho(i) \rho(j)}(h(\bm{u}_{i}, \bm{u}_{j})),
\end{split}    
\end{equation}
where $(k,m) \notin \{(i,j),(j,i)\}$.

In our assumption, $f_{S}(G)$ shows no pairwise statistical interaction between node $i$ and node $j$. That is, we can write $f_{S}(G)$ in the form of Equation \ref{eq_lemma_def1} according to Definition \ref{def:spi}. For the first component in the RHS of Equation \ref{eq_lemma_haseij}, we can easily construct functions $q_{\char`\\i}$ and $q_{\char`\\j}$ in a similar way of Equation \ref{eq_lemma_ignfun_average_fi} and Equation \ref{eq_lemma_ignfun_average_fj} respectively. However, for the second component in the RHS of Equation \ref{eq_lemma_haseij}, the non-additive function $h(\bm{u}_{i}, \bm{u}_{j})$ operates on node $i$ and node $j$. Through the definition of non-additive function, we cannot represent a non-additive function $h$ as a form like $h(\bm{u}_{i}, \bm{u}_{j})=f_{1}(\bm{u}_{i})+f_{2}(\bm{u}_{j})$, where $f_{1}$ and $f_{2}$ are functions. That is to say, we cannot merge the second component in the RHS into either $q_{\char`\\i}$ or $q_{\char`\\j}$. 

Therefore, Equation \ref{eq_lemma_haseij} cannot be represented as the form of Equation \ref{eq_lemma_def1}, and the node pair $(i,j)$ shows pairwise statistical interaction in $f_S(G)$, which contradicts our assumption. Lemma \ref{lemma_edge2inte} is proved.
\end{proof}

Combing Lemma \ref{lemma_inte2edge} and Lemma \ref{lemma_edge2inte}, Theorem \ref{thm:interaction} is proved.
\end{proof}

\section{Proof of Corollary \ref{coro:interaction}}
\label{appx:corollary}

\textbf{Corollary \ref{coro:interaction}.} \textit{(\textbf{Statistical Interaction in $L_0$-SIGN}) Consider a graph $G$ that the edge set is unknown. Let $G$ be the input of $L_0$-SIGN function $F_{LS}(G)$ in Equation \ref{fun:l0_final_function}, the function flags pairwise statistical interaction between node $i$ and node $j$ if and only if they are predicted to be linked by an edge in $G$ by $L_0$-SIGN, i.e., $e^{'}_{ij}=1$.}
\begin{proof}
In Equation \ref{fun:l0_final_function}, we can perform the prediction procedure by first predicting edge values on all potential node pairs. Then we perform node pair modeling and aggregating the results to get the predictions (as illustrated in Figure \ref{fig:sign_frame}). Specifically, we can regard the edge prediction procedure in $L_0$-SIGN as being prior to the following SIGN procedure. The edge values in an input graph $G(X, \emptyset)$ can be first predicted by function $F_{ep}$. Then, we have the graph $G(X, E^{'})$, where $E^{'}$ is a predicted edge set. Therefore, the following procedure is the same as the SIGN model with $G(X, E^{'})$ as the input graph, which satisfies Theorem \ref{thm:interaction}.
\end{proof}

\section{Algorithms}
\label{appx:algorithms}

In this section, we provide the pseudocode of SIGN and $L_0$-SIGN prediction algorithms in Algorithm \ref{alg:sign_prediction} and Algorithm \ref{alg:l0sign_prediction}, respectively. Meanwhile, we provide the pseudocode of SIGN and $L_0$-SIGN training algorithm in Algorithm \ref{alg:training_sign} and Algorithm \ref{alg:training_l0sign}, respectively.
\begin{algorithm}[H]
   \caption{SIGN prediction function $f_{S}$}
   \label{alg:sign_prediction}
\begin{algorithmic}
   \STATE {\bfseries Input:} data $G(X, E)$
   \FOR{each pair of feature ($i$, $j$)}
   \IF{$e_{ij} = 1$}
   \STATE $\bm{z_{ij}}=h(x_{i}\bm{v}_{i}, x_{j}\bm{v}_{j})$
   \STATE $\bm{s}_{ij}=\bm{z}_{ij}$
   \ELSE
   \STATE $\bm{s}_{ij}=\bm{0}$
   \ENDIF
   \ENDFOR
   \FOR{ each feature $i$}
   \STATE $\bm{v}_{i}^{'} = \psi(\varsigma_{i})$
   \STATE $\bm{u}_{i}^{'} = x_i\bm{v}_{i}^{'}$
   \STATE $\nu_i=g(\bm{u}_{i}^{'})$
   \ENDFOR
   \STATE $y^{'} = \phi(\nu)$
   \STATE {\bfseries Return:} $y^{'}$
\end{algorithmic}
\end{algorithm}

\begin{algorithm}[H]
   \caption{$L_0$-SIGN prediction function $f_{LS}$}
   \label{alg:l0sign_prediction}
\begin{algorithmic}
   \STATE {\bfseries Input:} data $G(X, \emptyset)$
   \FOR{each pair of feature ($i$, $j$)}
   \STATE $e^{'}_{ij}=HardConcrete(f_{ep}(\bm{v}^{e}_{i}, \bm{v}^{e}_{j}))$
   \STATE $\bm{z_{ij}}=h(x_{i}\bm{v}_{i}, x_{j}\bm{v}_{j})$
   \STATE $\bm{s}_{ij}=e^{'}_{ij}\bm{z}_{ij}$
   \ENDFOR
   \FOR{ each feature $i$}
   \STATE $\bm{v}_{i}^{'} = \psi(\varsigma_{i})$
   \STATE $\bm{u}_{i}^{'} = x_i\bm{v}_{i}^{'}$
   \STATE $\nu_i=g(\bm{u}_{i}^{'})$
   \ENDFOR
   \STATE $y^{'} = \phi(\nu)$
   \STATE {\bfseries Return:} $y^{'}$
\end{algorithmic}
\end{algorithm}

\begin{algorithm}[H]
   \caption{Training procedure of SIGN}
   \label{alg:training_sign}
\begin{algorithmic}
   \STATE Randomly initialize $\bm{\theta}$
   \REPEAT
   \FOR{ each input-output pair $(G_n(X_n, E_n), y_n)$}
   \STATE get $y^{'}_{n}, \bm{v}^{'}_{n}$ from $f_S(G_n; \bm{\theta})$
   \STATE $\bm{v}_{n} = \bm{v}^{'}_{n}$
   \ENDFOR
   \STATE $\mathcal{R}(\bm{\theta})=\frac{1}{N}\sum_{n=1}^{N}(\mathcal{L}(F_{LS}(y^{'}_{n},y_{n}))$
   \STATE update $\bm{\theta}$ (exclude $\bm{v}$) by $\min \mathcal{R}(\bm{\theta})$
   \UNTIL{ reach the stop conditions}
\end{algorithmic}
\end{algorithm}

\begin{algorithm}[H]
   \caption{Training procedure of $L_0$-SIGN}
   \label{alg:training_l0sign}
\begin{algorithmic}
   \STATE Randomly initialize $\bm{\theta}, \bm{\omega}$
   \REPEAT
   \FOR{each input-output pair $(G_n(X_n, \emptyset),y_n)$}
   \STATE get $y^{'}_{n}, \bm{v}^{'}_{n}$ from $f_{LS}(G_n; \bm{\theta}, \bm{\omega})$
   \STATE $\bm{v}_{n} = \bm{v}^{'}_{n}$
   \ENDFOR
   \STATE calculate $\mathcal{R}(\bm{\theta}, \bm{\omega})$ through Equation \ref{fun:l0_sign_loss}
   \STATE update $\bm{\omega}, \bm{\theta}$(exclude $\bm{v}$) by $\min \mathcal{R}(\bm{\theta}, \bm{\omega})$
   \UNTIL{ reach the stop conditions}
\end{algorithmic}
\end{algorithm}

\section{Time Complexity Analysis}
\label{appx:time_complexity}
Our model performs interaction detection on each pair of features and interaction modeling on each detected feature interaction. The time complexity is $O(q^2(M^e+M^h))$, where $q=|X_n|$, $O(M^e)$ and $O(M^h)$ are the complexities of interaction detection and interaction modeling on one feature interaction, respectively. The two aggregation procedures $\phi$ and $\psi$ can be regarded together as summing up all statistical interaction analysis results two times. The complexity is $O(2dq^2)$. Finally, there is an aggregation procedure $g$ on each node. The complexity is $O(qd)$. Therefore, the complexity of our model is $O(q^2(M^e+M^h+2d) + qd)$.

\section{$L_0$ Regularization}
\label{appx:l0_reg}
$L_{0}$ regularization encourages the regularized parameters $\bm{\theta}$ to be exactly zero by setting an $L_0$ term:
\begin{equation}
\label{eq_l0}
   \lVert \bm{\theta} \rVert_{0} =\sum_{j=1}^{|\bm{\theta}|}\mathbb{I} [\theta_{j}\ne 0],
\end{equation}
where $|\bm{\theta}|$ is the dimensionality of the parameters and $\mathbb{I}$ is $1$ if $\theta_{j}\ne 0$, and $0$ otherwise.

For a dataset $D$, an empirical risk minimization procedure is used with $L_{0}$ regularization on the parameters $\bm{\theta}$ of a hypothesis $\mathcal{H}(\cdot;\bm{\theta})$, which can be any objective function involving parameters, such as neural networks. Then, using reparameterization of $\bm{\theta}$ , we set $\theta_j=\tilde{\theta}_j z_j$, where $\tilde{\theta}_j\ne 0$ and $z_j$ is a binary gate with Bernoulli distribution $Bern(\pi_j)$ \cite{louizos2017learning}. The procedure is represented as:
\begin{equation}
\small
\begin{gathered}
    \label{eq_l0_rmp}
   \mathcal{R}(\bm{\tilde{\theta}}, \bm{\pi}) = \mathbb{E}_{p(\bm{z}\mid \bm{\pi})} \frac{1}{N}(\sum_{n=1}^{N}\mathcal{L}(\mathcal{H}(X_n;\bm{\tilde{\theta}}\odot\bm{z}),y_n)) + \lambda \sum_{j=1}^{|\bm{\theta}|}\pi_j,  \\
\bm{\tilde{\theta}}^{*}, \bm{\pi}^{*}=\argmin_{\bm{\tilde{\theta}},\bm{\pi}} {\mathcal{R}(\bm{\tilde{\theta}}, \bm{\pi})}, 
\end{gathered}
\end{equation}
where $p(z_j|\pi_j)=Bern(\pi_j)$, $N$ is the number of samples in $D$, $\odot$ is element-wise production, $\mathcal{L}(\cdot)$ is a loss function and $\lambda$ is the weighting factor of the $L_{0}$ regularization.

\section{Approximate $L_0$ Regularization with Hard Concrete Distribution}
\label{appx:hard_concrete}
A practical difficulty of performing $L_0$ regularization is that it is non-differentiable. Inspired by \cite{louizos2017learning}, we smooth the $L_0$ regularization by approximating the binary edge value with a hard concrete distribution. 
Specifically, let $f_{ep}$ now output continuous values. Then
\begin{equation}
\label{fun:hard_concrete}
\small
\begin{split}
     & u \sim \mathcal{U}(0,1), \\
     & s  = Sigmoid((\log u-\log (1-u)+\log(\alpha_{ij}))/\beta), \\
     & \Bar{s}  = s(\delta-\gamma) + \gamma, \\
     & e^{'}_{ij}  = min(1, max(0, \Bar{s})),
\end{split}
\end{equation}
where $u\sim\mathcal{U}(0,1)$ is a uniform distribution, $Sig$ is the Sigmoid function, $\alpha_{ij}\in \mathbb{R}^{+}$ is the output of $f_{ep}(\bm{v}^{e}_{i},\bm{v}^{e}_{j})$, $\beta$ is the temperature and $(\gamma, \delta)$ is an interval with $\gamma<0$, $\delta>0$.

Therefore, the $L_0$ activation regularization is changed to: $\sum_{i,j\in X_n}(\pi_{n})_{ij} = \sum_{i,j\in X_n}Sig(\log \alpha_{ij} - \beta \log \frac{-\gamma}{\delta}).$

As a result, $e^{'}_{ij}$ follows a hard concrete distribution through the above approximation, which is differentiable and approximates a binary distribution. Following the recommendations from \cite{maddison2016concrete}, we set $\gamma=-0.1, \delta=1.1$ and $\beta=2/3$ throughout our experiments.
We refer interested readers to \cite{louizos2017learning} for details about hard concrete distributions.

\section{Additional Experimental Results}
\label{appx:f1-score}
We also evaluate our model using the f1-score. We show the evaluation results in this section.

\begin{table}[H]
\caption{Summary of results in comparison with baselines.}
\label{appxtab:performance_recom}
\begin{center}
\begin{sc}
\begin{tabular}{l|cc}
\hline
 & \textbf{Frappe} & \textbf{MovieLens} \\
 & F1-score & F1-score \\
\hline
FM      & 0.8443 & 0.8374\\
AFM     & 0.8533 & 0.8531 \\
NFM     & 0.8610 & 0.8624 \\
DeepFM  & 0.8671 & 0.8577 \\
xDeepFM & 0.8742 & 0.8681 \\
AutoInt & 0.8758 & 0.8693 \\
\hline
SIGN  & 0.8869 & 0.8704 \\
$L_0$-SIGN  & \textbf{0.8906} & \textbf{0.8771}\\
\hline
\end{tabular}
\end{sc}
\end{center}
\end{table}

\begin{table}[H]
\caption{The results in comparison with existing GNNs.}
\label{appxtab:performance_graph}
\begin{center}
\begin{sc}
\begin{tabular}{l|cc}
\hline
\multirow{2}{*}{} & \textbf{Twitter} & \textbf{DBLP} \\
 & F1-score & F1-score \\
\hline
GCN  & 0.6527 & 0.9198  \\
$L_0$-GCN & 0.6523 & 0.9208 \\
\hline
CHEBY & 0.6491 & 0.9192 \\
$L_0$-CHEBY & 0.6495 & 0.9197 \\
\hline
GIN & 0.6567 & 0.9227 \\
$L_0$-GIN & 0.6573 & 0.9234 \\
\hline
SIGN & 0.6601 & 0.9275 \\
$L_0$-SIGN & \textbf{0.6674} & \textbf{0.9397} \\
\hline
\end{tabular}
\end{sc}
\end{center}
\end{table}

We can see that our model still performs best among all baselines when evaluating with the F1-score.

\end{document}